%% file: main.tex
\documentclass{article} 
\usepackage{iclr2027_conference,times}

\input{math_commands.tex}

\usepackage{url}

\title{Dimension-Free Scaling Laws for Size \\Generalization:
An Exponential Separation}

\author{Behrooz Tahmasebi \& Melanie Weber \\
Harvard John A.~Paulson School of Engineering and Applied Sciences (SEAS)\\
Harvard  University, Cambridge, MA 02138, USA \\
\texttt{\{behrooz\_tahmasebi,mweber\}@seas.harvard.edu} 
}

\iclrfinalcopy 

\usepackage[utf8]{inputenc} 
\usepackage[T1]{fontenc}    
\usepackage{url}            
\usepackage{booktabs}       
\usepackage{array}          
\usepackage{wrapfig}        
\usepackage{caption}        
\usepackage{amsfonts}       
\usepackage{nicefrac}       
\usepackage{microtype}      
\usepackage{xcolor}         

\definecolor{darkblue}{rgb}{0.0,0.0,0.65}
\definecolor{darkred}{rgb}{0.68,0.05,0.0}
\definecolor{darkgreen}{rgb}{0.0,0.29,0.29}
\definecolor{darkpurple}{rgb}{0.47,0.09,0.29}
 
\usepackage[backref=page]{hyperref}
\hypersetup{
   colorlinks = true,
   citecolor  = darkblue,
   linkcolor  = darkred,
   filecolor  = darkblue,
   urlcolor   = darkblue,
 }
 \usepackage[capitalize,noabbrev]{cleveref}

\usepackage{amsthm,amsmath,amssymb}

\usepackage{mathtools}

\theoremstyle{plain}
\newtheorem{theorem}{Theorem}[section]

\newtheorem{lemma}[theorem]{Lemma}
\newtheorem{corollary}[theorem]{Corollary}

\theoremstyle{definition}

\theoremstyle{remark}

\usepackage[most]{tcolorbox}

\usepackage{enumitem}

\begin{document}
\addtocontents{toc}{\protect\setcounter{tocdepth}{-1}}

\maketitle
\lhead{Preprint}

\begin{abstract}
Modern machine learning models increasingly operate on variable-size structured data, such as sets and graphs. A central requirement is \emph{size generalization}: the ability of a model trained on inputs of one size to generalize to inputs of a different, potentially much larger, size. Classical statistical learning theory, however, primarily considers fixed input domains, and its guarantees often deteriorate with input dimension. Although extensive work has established when models on sets and graphs can be transferred across input sizes, such results do not characterize the statistical cost of doing so: the sample complexity of size generalization, and its relationship to classical in-domain learning, remain largely unknown. We develop a statistical learning theory of size generalization and establish a sharp exponential separation in sample complexity. Letting \(n\) denote the number of training samples, we show that polynomial source decay is insufficient for efficient size generalization: while it yields error \(n^{-\Theta(1)}\) and hence polynomial sample complexity for in-domain learning, size-generalization error decays only as \((\log n)^{-\Theta(1)}\), yielding an exponential separation. Thus, polynomial in-domain sample complexity can become exponential across sizes. We then identify an exponential source condition under which size generalization achieves error \(n^{-\Theta(1)}\), recovering polynomial sample complexity. These guarantees are independent of input dimension and depend only on the number of training samples, yielding dimension-free scaling laws. We also provide matching minimax lower bounds, establishing optimality. To our knowledge, this work provides the first dimension-free minimax theory of size generalization and a rigorous characterization of when, and at what statistical cost, learning can transfer across input sizes.
\end{abstract}

\section{Introduction}

In many geometric learning problems, the domain is not fixed once and for all. Sets vary in cardinality, graphs vary in the number of nodes, and objects encountered at test time may be larger than those observed during training. The goal is therefore not only to generalize to unseen in-domain samples but also to generalize across domains of different sizes. This desirable property is central to geometric machine learning \citep{bronstein2021geometric,weber2025geometric} and is often studied as \emph{size generalization} for sets and graphs \citep{yehudai2021local,bevilacqua2021size}.

Recent work has made important progress toward understanding learning across sizes, but the available theory remains limited. One line studies transferability through graph limits, graphons, and related stability frameworks \citep{ruiz2020graphon,keriven2020convergence,levie2021transferability,levin2025transferring}, while another develops learning rules across dimensions using representation stability \citep{farb2014representation,levin2024any,levin2023free,diaz2025invariant}. These works offer structural and statistical insights, but leave open a dimension-free minimax characterization of sample complexity across sizes.


The ability to process variable-size inputs does not by itself close this gap. Deep Sets \citep{zaheer2017deep} and graph neural networks \citep{scarselli2008graph}, for example, can be applied to sets and graphs of varying sizes and may even satisfy transferability guarantees: under suitable regularity conditions, such as Lipschitz continuity, a fixed model produces similar outputs on objects of different sizes. But such guarantees do not determine how the transferred error decreases with the number \(n\) of training samples, whether this rate is uniform over the training and test sizes, or how it compares with generalization at the training size. We therefore ask: given \(n\) samples at size \(d\), can we control the generalization error simultaneously at every target size \(D\ge d\), using a number of samples independent of both \(d\) and \(D\)?


To address this question, we return to a classical idea. In classical learning theory, it is standard to impose a \emph{source condition} on the desired task, since learning over an arbitrary space is impossible by a standard no-free-lunch argument. The source condition that enables the in-domain generalization error to decrease as \(n^{-\Theta(1)}\) is called polynomial source decay. Roughly speaking, it requires the Fourier coefficients of the target function to decay polynomially with their index. This condition is necessary and sufficient for efficient classical in-domain learning. Motivated by this phenomenon, we ask the following question:
\begin{tcolorbox}
\centering
\emph{Do the conditions for efficient in-domain learning also enable efficient size generalization?}
\end{tcolorbox}

\begin{wraptable}{r}{0.56\textwidth}
    \centering
    \small
    \captionsetup{font=small,labelfont=bf,skip=3pt}
    \caption{The statistical cost of size generalization.
    }
    \label{tab:size-generalization-rates}
    \begin{tabular}{@{}lcc@{}}
    \toprule
    Source decay & In-domain error & Size-generalization error \\
    \midrule
    Polynomial & \(n^{-\Theta(1)}\) & \((\log n)^{-\Theta(1)}\) \\
    Exponential & \(n^{-\Theta(1)}\) & \(n^{-\Theta(1)}\) \\
    \bottomrule
    \end{tabular}
    \vspace{-0.8em}
    \end{wraptable}

We therefore ask whether polynomial source decay remains sufficient for size generalization. Our main finding is a sharp separation: under polynomial source decay, the size-generalization error decreases only as \((\log n)^{-\Theta(1)}\). Consequently, an accuracy attainable with polynomially many samples for in-domain learning under the same condition may require exponentially many samples when the task must generalize across input sizes.


Importantly, our resulting guarantees are \emph{dimension-free}: they depend on the number of training samples, but not on the input size \(d\) (i.e., either the cardinality of a set or the number of nodes in a graph).
In particular, classical transferability errors are typically controlled by the smaller object size, scaling as \(d^{-\Theta(1)}\) when \(D\ge d\). By contrast, our generalization gap converges to zero with the number of training samples \(n\) and is uniform over both \(d\) and \(D\), albeit at a slow rate.

To overcome this limitation, we then identify when the exponential cost can be avoided. Under \emph{exponential source decay}, which is stronger than polynomial source decay, we show that the size-generalization error improves from \((\log n)^{-\Theta(1)}\) to \(n^{-\Theta(1)}\), recovering polynomial sample complexity. Therefore, our main finding is:
\begin{tcolorbox}
\centering
\emph{Efficient size generalization requires stronger regularity than in-domain learning; polynomial decay is sufficient for the latter, whereas exponential decay is needed for the former.}
\end{tcolorbox}
Table~\ref{tab:size-generalization-rates} summarizes this contrast. Exponential decay arises naturally for heat (Gaussian) kernels, analytic functions, and local graph functions.

Finally, we prove matching minimax lower bounds for both sets and graphs. The polynomial-versus-exponential separation is therefore not caused by a particular algorithmic issue. It is an inherent statistical limitation of size generalization over the corresponding source classes. Together, the upper and lower bounds provide a dimension-free minimax characterization of when learning can generalize across input sizes and at what sample cost.

In short, in this paper we make the following contributions:
\begin{itemize}[leftmargin=1.5em]
    \item We develop dimension-free scaling laws for invariant learning and estimation on sets and graphs, with rates that hold uniformly across input sizes (Section~\ref{sec:main-results} and Theorem~\ref{thm:size-generalization-oracle}).

    \item We show that size generalization is fundamentally harder than classical in-domain statistical learning: under polynomial source decay, in-domain error is \(n^{-\Theta(1)}\), while size-generalization error is only \((\log n)^{-\Theta(1)}\), implying an exponential separation in sample complexity (Theorem~\ref{thm:size-generalization-oracle} and Corollary~\ref{cor:sample-complexity-separation}).

    \item We show that under exponential source decay, size-generalization error becomes \(n^{-\Theta(1)}\), recovering polynomial sample complexity independently of object size, at the cost of stronger regularity (Theorem~\ref{thm:size-generalization-oracle}).

    \item We prove that the rates for both sets and graphs are minimax optimal, establishing that the separation is intrinsic rather than algorithm-dependent (Theorem~\ref{thm:minimax-size-generalization}).
\end{itemize}

\section{Related Work}

\paragraph{Transferability and size generalization.}
Transferability across graph sizes has been studied through graph limits, graphons, graphops, and higher-order limits \citep{ruiz2020graphon,keriven2020convergence,levie2021transferability,ruiz2023transferability,maskey2023transferability,cai2022convergence,le2023limits,herbst2025higher}. Related work studies learning across set cardinalities and learning-theoretic formulations of graph size generalization \citep{zweig2021functional,maehara2021learning}. These works primarily characterize approximation, consistency, or stability across sizes. In contrast, we ask when an estimator learned at one size achieves statistical convergence rates that are uniform over all test sizes.

\paragraph{Representation stability and any-dimensional learning.}
Our work builds on representation stability \citep{farb2014representation}, which describes how structured families of representations stabilize across dimensions. Representation stability has recently been used in machine learning to design and analyze architectures for learning under symmetries \citep{levin2024any}, optimization across dimensions \citep{levin2023free}, graph limits \citep{levin2026limitsweight}, scientific computing and PDEs \citep{phan2025towards}, and finite-dimensional kernel learning under invariances \citep{diaz2025invariant}.
The work of \citet{levin2025transferring}, which establishes size-generalization bounds in a broad framework encompassing sets and graphs, is most closely related to ours. Although their framework provides transfer and generalization guarantees across dimensions, the resulting rates depend on dimension-dependent complexity measures, such as covering numbers of the domains. In contrast, we identify conditions under which convergence rates are independent of object size and characterize the tradeoff between invariant feature dimension growth and target approximation. To the best of our knowledge, our work is the first to establish this sharp separation between slow and fast regimes together with matching minimax optimality guarantees. For a broader discussion of related work, see the appendix.

\section{Problem Statement}\label{sec:problem-statement}

\paragraph{Sets and (attributed) graphs.}
We study problems indexed by the object size \(d\), the number of elements in a set or nodes in a graph, and fix a feature dimension \(r\in\mathbb N\). A set is represented by \(X=(x_1,\ldots,x_d)\in(\mathbb R^r)^d\simeq\mathbb R^{r\times d}\), while an attributed graph is represented by \(X\in\mathbb R^{r\times d\times d}\), with \(r\) edge and/or node-feature channels. Accordingly, \(\mathcal X_d\) denotes \(\mathbb R^{r\times d}\) or \(\mathbb R^{r\times d\times d}\). These matrix and tensor representations are ordered, but the underlying sets and graphs are not. We equip \(\mathcal X_d\) with the standard Gaussian probability measure \(\lambda_d\), under which the vectorized entries are independent \(\mathcal N(0,1)\). This measure is permutation invariant and has positive Lebesgue density everywhere, so it can serve as a reference for any distribution with a Lebesgue density; we choose it because Hermite polynomials form a convenient orthogonal basis.
 
\paragraph{Statistical learning at a fixed size.}
We consider unsupervised density estimation and supervised regression together. Write \(Z_i=X_i\) in the unsupervised setting, where \(X_i\overset{\mathrm{i.i.d.}}{\sim}\mu_d\) and \(d\mu_d=f_d^\star d\lambda_d\), and write \(Z_i=(X_i,Y_i)\) in the supervised setting, where \((X_i,Y_i)\overset{\mathrm{i.i.d.}}{\sim}\rho_d\), \(Y_i\in\mathbb R\), the \(X\)-marginal is \(\mu_d\), and \(f_d^\star(x)=\mathbb E[Y\mid X=x]\). Let \(\mathcal H_d=L^2(\lambda_d)\) for density estimation and \(\mathcal H_d=L^2(\mu_d)\) for regression, and assume \(f_d^\star\in\mathcal F_d\subseteq\mathcal H_d\) is the target function to estimate. For the corresponding loss \(\ell_d(Z,f)\), define
\[
\widehat L_{n,d}(f)=\frac1n\sum_{i=1}^n\ell_d(Z_i,f),
\qquad
L_d(f)=\mathbb E[\ell_d(Z,f)].
\]
We assume that \(\ell_d\) identifies \(f_d^\star\): the target has finite population risk and is the unique minimizer of \(L_d\) over \(\mathcal F_d\) as an element of \(L^2\). This uniqueness is functional, not parametric. Many parameter values may represent the same minimizing function, so the assumption is fully compatible with overparameterized models. For example, negative log-likelihood identifies the true density, while squared loss identifies the conditional mean whenever \(\mathbb E[Y^2]<\infty\); more generally, proper losses identify their associated population targets. Because \(\mu_d\) or \(\rho_d\) is unknown, \(L_d\) cannot be evaluated directly, and empirical risk minimization provides its data-driven counterpart.

Although we state the main formulation for scalar invariant targets, the framework also covers permutation-equivariant vector-valued targets after replacing the invariant spaces and feature counts by their equivariant counterparts; Appendix~\ref{app:m-estimation-examples} gives score matching and denoising as examples. This extension is at the level of function spaces: it requires exact norm-controlled transfer, stable equivariant feature growth, and the corresponding source condition, and its rates are governed by the relevant equivariant feature counts.

\paragraph{Permutation invariance.}
Let \(S_d\) be the symmetric group. A permutation \(\pi\in S_d\) acts on sets by \((\pi\cdot X)_{:,j}=X_{:,\pi^{-1}(j)}\) and on graphs by \((\pi\cdot X)_{a,i,j}=X_{a,\pi^{-1}(i),\pi^{-1}(j)}\). Because elements and nodes are unlabeled, we assume \(\pi\cdot X\overset{\mathrm{law}}{=}X\) for density estimation and \((\pi\cdot X,Y)\overset{\mathrm{law}}{=}(X,Y)\) for regression. Invariance of \(\mu_d\) and \(\lambda_d\), together with uniqueness of the density, gives \(f_d^\star(\pi\cdot x)=f_d^\star(x)\) almost everywhere in the first case; invariance of the conditional law gives the same conclusion for the regression function. Thus \(\mathcal F_d^{\mathrm{inv}}\coloneqq\{f\in\mathcal F_d:f(\pi\cdot x)=f(x)\text{ a.e. for every }\pi\in S_d\}\), and the invariant empirical risk minimizer is
\[
\hat f_{n,d}^{\mathrm{inv}}
\in
\arg\min_{f\in\mathcal F_d^{\mathrm{inv}}}
\widehat L_{n,d}(f).
\]
Appendix~\ref{app:group-actions} provides the formal background.

\paragraph{Exact functional transfer across sizes.}
Let \(\mathcal H_d^{\mathrm{inv}}\subseteq\mathcal H_d\) be the Hilbert space of invariant functions containing \(\mathcal F_d^{\mathrm{inv}}\), with the corresponding \(L^2\) norm \(\|\cdot\|_d\). For \(d\le D\), let
\[
T_{d\to D}:\mathcal H_d^{\mathrm{inv}}\longrightarrow
\mathcal H_D^{\mathrm{inv}}
\]
be a linear operator satisfying \(T_{d\to D}(\mathcal F_d^{\mathrm{inv}})\subseteq\mathcal F_D^{\mathrm{inv}}\). It maps an entire rule, independently of its parameterization, from size \(d\) to size \(D\). We assume coherence, \(T_{d\to d}=\mathrm{Id}\) and \(T_{D\to E}\circ T_{d\to D}=T_{d\to E}\) for \(d\le D\le E\); exact target transfer, \(T_{d\to D}f_d^\star=f_D^\star\); and norm preservation, \(\|T_{d\to D}f-T_{d\to D}h\|_D=\|f-h\|_d\).
Thus the targets at different sizes are realizations of a single size-consistent rule, with no transfer mismatch or approximation error between sizes. This consistency assumption is standard and essential in the size-generalization literature \citep{levin2025transferring}.

This is the function-level abstraction of the shared rules used by Deep Sets, GNNs, invariant feature expansions, and graphon representations \citep{zaheer2017deep,xu2018powerful,ruiz2020graphon,levin2024any,levin2025transferring}. The map need not cover every function at size \(D\); it only carries the same rule to the larger domain while preserving its error scale. Without such a connection, size-\(d\) data would contain no information about an arbitrary size-\(D\) target. Appendices~\ref{app:any-dimensional-models} and~\ref{app:sets-graphs-examples} give formal definitions and concrete constructions.

\paragraph{A concrete picture.}
An invariant rule can be viewed as a weighted combination of patterns whose meaning does not change with object size. For a set, a pattern might measure an average feature or an interaction among a few elements; for a graph, it might measure a node or edge feature or the frequency of a small subgraph. Let \(\psi_{d,j}\) denote the normalized measurement of pattern \(j\) on an object of size \(d\). The same index \(j\) refers to the same pattern at every size. A rule \(\sum_j a_j\psi_{d,j}\) is therefore transferred to size \(D\) simply by reusing the learned weights:
\[
T_{d\to D}\!\left(\sum_j a_j\psi_{d,j}\right)
=\sum_j a_j\psi_{D,j}.
\]
With a suitable normalization, the squared difference between two rules is exactly the squared difference between their weights. Reusing the weights at a new size therefore preserves prediction error; formally, this is the orthonormal-basis construction and Parseval's identity. Exact target transfer then has a simple meaning: the true task uses the same pattern weights \((a_j)_j\) at every size. Thus, the assumption does not identify arbitrary functions across dimensions; it only identifies the same invariant rule evaluated on different numbers of elements or nodes. Appendices~\ref{app:any-dimensional-models} and~\ref{app:sets-graphs-examples} construct these features explicitly for sets and graphs.

\paragraph{The statistical objective: generalization across size.}
Training uses only \(n\) samples at size \(d\). At any target size \(D\ge d\), the learned function is \(T_{d\to D}\hat f_{n,d}^{\mathrm{inv}}\). Our goal is a guarantee whose sample complexity is independent of both sizes:
\begin{tcolorbox}
\textbf{Dimension-free size generalization.}
For every accuracy \(\varepsilon>0\) and confidence level \(1-\delta\), find \(n=n(\varepsilon,\delta)\), independent of \(d\) and \(D\), such that, with probability at least \(1-\delta\),
\[
\sup_{D\ge d}
\bigl\|
T_{d\to D}\hat f_{n,d}^{\mathrm{inv}}-f_D^\star
\bigr\|_D^2
\le\varepsilon.
\]
\end{tcolorbox}
The supremum is essential: a single estimator, learned at one size, must remain accurate simultaneously at every larger size. The remainder of the paper characterizes exactly when this is statistically possible and how the required number of samples scales with the regularity of the target.

\section{Main Results}
\label{sec:main-results}

We now give the dimension-free rates for statistical learning on sets and graphs. The result has two ingredients: representation stability determines how many invariant interactions exist at each order, while target regularity determines how much of the learning task remains beyond that order. Combining them yields a sharp phase transition between nearly parametric and logarithmic size generalization.

\subsection{Invariant polynomial coordinates and representation stability}

Recall the domain \(\mathcal X_d\), target class \(\mathcal F_d\), and Hilbert space \(\mathcal H_d\) from Section~\ref{sec:problem-statement}. Thus, \(\mathcal F_d\subseteq\mathcal H_d^{\mathrm{inv}}\), with \(\mathcal H_d=L^2(\lambda_d)\) for density estimation and \(\mathcal H_d=L^2(\mu_d)\) for regression.

To organize this infinite-dimensional space, let \(\mathcal V_{d,k}\) be the span of all invariant polynomials of degree at most \(k\) on \(\mathcal X_d\). Under the standard completeness conditions stated in the appendix, polynomials are dense in \(\mathcal H_d\); averaging over relabelings then makes invariant polynomials dense in \(\mathcal H_d^{\mathrm{inv}}\). The polynomial description therefore does not assume that the target itself is a polynomial. It supplies a complete, transfer-compatible hierarchy in which degree records interaction complexity: low degrees describe low-order interactions, while increasing \(k\) progressively resolves richer dependencies among elements or nodes.

An orthonormal basis for \(\mathcal V_{d,k}\) is obtained by applying Gram--Schmidt in \(\mathcal H_d\), with inner product \(\langle\cdot,\cdot\rangle_d\).
Orthogonalization changes only the coordinates, not the polynomial span or its dimension. The resulting nested spaces
\[
\mathcal V_{d,0}\subseteq\mathcal V_{d,1}\subseteq\cdots
\subseteq\mathcal H_d^{\mathrm{inv}}
\]
are compatible with the exact transfer operators, and representation stability controls their complexity uniformly over object size. Specifically, for \(\mathsf X\in\{\mathrm{set},\mathrm{graph}\}\), define
\[
q_{\mathsf X}(k)
\coloneqq
\sup_{d\ge1}\dim(\mathcal V_{d,k}).
\]
For every fixed \(k\), the dimensions increase to a stable value, so the supremum equals that value and bounds every size, including the pre-stable regime \citep{farb2014representation}.

For sets \(X\in\mathbb R^{r\times d}\), these are multisymmetric polynomials. Hardy--Ramanujan and Meinardus partition asymptotics, detailed in Appendix~\ref{app:sets-multisymmetric-polynomials}, give
\[
q_{\mathrm{set}}(k)
=
\exp\!\left(C_r k^{\beta}+o(k^\beta)\right),
\qquad
\beta=\frac{r}{r+1}.
\]
Here \(C_r=\frac{r+1}{r}\bigl(r\zeta(r+1)\bigr)^{1/(r+1)}\), where \(\zeta\) is the Riemann zeta function.
For attributed graphs \(X\in\mathbb R^{r\times d\times d}\) with at least one edge channel, invariant polynomials are indexed by unlabeled attributed multigraph patterns. The enumeration in Appendix~\ref{app:attributed-graph-invariants} gives
\[
q_{\mathrm{graph}}(k)
=
\exp\!\left(k\log k+O_r(k\log\log k)\right).
\]
For fixed \(r\), the graph feature dimension depends on \(r\) only through the lower-order term \(O_r(k\log\log k)\). Consequently, \(r\) does not enter the first-order graph asymptotics or the graph rates below. Both feature counts are uniform in the number of elements or nodes, but the number of interactions grows rapidly with degree and substantially faster for graphs.

\subsection{Polynomial and exponential source decay}

Let \(\mathcal F_{d,k}=\mathcal V_{d,k}\) for regression and \(\mathcal F_{d,k}=\mathcal F_d^{\mathrm{inv}}\cap\mathcal V_{d,k}\) for density estimation, so density approximants remain nonnegative and normalized. Define the best order-\(k\) approximation error and its uniform envelope by
\[
B_d(k)
\coloneqq
\inf_{f\in\mathcal F_{d,k}}\|f_d^\star-f\|_d^2,
\qquad
B(k)
\coloneqq
\sup_{d\ge1}B_d(k).
\]
For regression, the optimizer is the orthogonal projection \(\Pi_{d,k}f_d^\star\); hence \(B(k)\) is the squared energy in coordinates of degree greater than \(k\), uniformly over \(d\). The density definition is its admissible analogue. In both cases, this is an invariant Fourier or spectral source condition: the basis is complete, while coefficient decay quantifies how strongly the target depends on high-order interactions.

We use the two source conditions introduced in the paper:
\[
\begin{array}{lll}
\textit{polynomial source decay:}
& B(k)\lesssim k^{-\tau}, & \tau>0,\\[2pt]
\textit{exponential source decay:}
& B(k)\lesssim \exp(-a k^\nu), & a,\nu>0.
\end{array}
\]
Polynomial decay is the analogue of finite smoothness: increasingly complex interactions contribute progressively less, but remain present at every order. Exponential decay describes substantially stronger regularity. The exponent \(\nu\) determines whether this decay is fast enough to overcome the growth of invariant interactions.

\subsection{Dimension-free scaling laws}

For an interaction order \(k\), let \(f_{d,k}^{\circ}\) minimize the population risk over \(\mathcal F_{d,k}\), and let \(\hat f_{n,d,k}^{\mathrm{inv}}\) be its empirical counterpart. We impose standard, primitive regularity conditions on this finite-dimensional problem. In orthonormal coordinates on \(\mathcal V_{d,k}\), the population approximation satisfies \(\|f_{d,k}^{\circ}-f_d^\star\|_d^2\lesssim B_d(k)\); the empirical loss has curvature at least \(\alpha>0\) on the relevant neighborhood; and the centered empirical score at \(f_{d,k}^{\circ}\) is sub-Gaussian with parameter at most \(\sigma^2\). The constants \(\alpha\) and \(\sigma\) are uniform in \(d\) and \(k\). Here, curvature means that moving the fitted function by \(h\) increases the loss by at least \(\alpha\|h\|_d^2/2\), up to its linear term; it rules out statistically invisible directions.

These conditions are directly checkable for the canonical losses in our framework. For squared-loss regression,
\[
L_d(f)-L_d(f_d^\star)=\|f-f_d^\star\|_{L^2(\mu_d)}^2,
\]
so approximation is exactly measured by \(B_d(k)\). Curvature follows from concentration of the empirical Gram matrix, and score concentration follows from sub-Gaussian regression noise and whitened invariant features. For density estimation by negative log-likelihood, the Hessian is uniformly positive and the likelihood is locally equivalent to squared \(L^2(\lambda_d)\) error when candidate densities are bounded above and away from zero; bounded or sub-Gaussian invariant scores then give the same concentration. Regular proper losses satisfy the analogous calibration and Hessian conditions. Appendix~\ref{app:proof-invariant-m-estimation} derives the finite-dimensional estimation bound from these assumptions rather than postulating it.

Table~\ref{tab:main-minimax-rates} compares our dimension-free rates with classical in-domain learning. At a fixed size \(d\), the number of degree-at-most-\(k\) invariant polynomial features grows as \(k^{rd}\) for sets and \(k^{rd^2}\) for graphs under our tensor representation; invariance under the finite group \(S_d\) changes constants but not these exponents. Constants and lower-order factors are suppressed throughout.

\begin{table}[t]
\centering
\small
\setlength{\tabcolsep}{4pt}
\renewcommand{\arraystretch}{1.35}
\caption{Minimax-optimal in-domain and dimension-free size-generalization rates.}
\label{tab:main-minimax-rates}
\begin{tabular}{@{}l>{\centering\arraybackslash}p{0.22\linewidth}>{\centering\arraybackslash}p{0.23\linewidth}>{\centering\arraybackslash}p{0.23\linewidth}@{}}
\toprule
\textbf{Source decay} & \textbf{In-domain} & \textbf{Sets} & \textbf{Graphs} \\
\midrule
Polynomial, \(k^{-\tau}\)
& Sets: \(n^{-\tau/(\tau+rd)}\)\newline Graphs: \(n^{-\tau/(\tau+rd^2)}\)
& \((\log n)^{-\tau(r+1)/r}\)
& \((\log n/\log\log n)^{-\tau}\) \\
Exponential, \(e^{-ak^\nu}\)
& Both: \(n^{-1+o(1)}\) 
& \(n^{-1+o(1)}\) if \(\nu>\frac{r}{r+1}\)
& \(n^{-1+o(1)}\) if \(\nu>1\) \\
\bottomrule
\end{tabular}
\end{table}

At and below the exponential thresholds, the rates interpolate between the two rows of the table. The exact critical and subcritical expressions are derived in Appendix~\ref{app:proof-scaling-laws}; the main distinction is whether exponential source decay is fast enough to recover a nearly parametric rate.

\begin{theorem}[Dimension-free scaling laws for sets and graphs]
\label{thm:size-generalization-oracle}
\label{thm:invariant-m-estimation}
Assume exact isometric transfer, representation stability, compatibility \(T_{d\to D}(\mathcal F_{d,k})\subseteq\mathcal F_{D,k}\), and the uniform loss regularity above. For any fixed confidence level, there are invariant empirical estimators, trained using only size-\(d\) data, whose squared error holds simultaneously for every \(D\ge d\) and satisfies the following rates, independently of \(d\) and \(D\):
\begin{enumerate}[label=\textup{(\roman*)},leftmargin=2.1em]
    \item If \(B(k)\lesssim k^{-\tau}\), the rate is
    \[
    (\log n)^{-\tau(r+1)/r}\quad\text{for sets},
    \qquad
    \left(\frac{\log n}{\log\log n}\right)^{-\tau}\quad\text{for graphs}.
    \]
    \item If \(B(k)\lesssim e^{-ak^\nu}\), the set rate, with \(\beta=r/(r+1)\), is
    \[
    \begin{cases}
    n^{-1+o(1)}, & \nu>\beta,\\
    n^{-a/(a+C_r)+o(1)}, & \nu=\beta,\\
    \exp\!\left[-aC_r^{-\nu/\beta}(\log n)^{\nu/\beta}(1+o(1))\right], & 0<\nu<\beta,
    \end{cases}
    \]
    while the graph rate is \(n^{-1+o(1)}\) for \(\nu>1\) and
    \[
    \exp\!\left[-a\left(\frac{\log n}{\log\log n}\right)^\nu(1+o(1))\right]
    \quad\text{for }0<\nu\le1.
    \]
\end{enumerate}
The statements apply to squared-loss regression, regular likelihood density estimation, and other proper losses satisfying the stated calibration, curvature, and score conditions.
\end{theorem}

The theorem is uniform in a strong sense: one event controls the transferred estimator at every larger deployment size, without a union bound over \(D\). Representation stability removes object size from the rate, but it does not make all invariant interactions equally easy to learn. For polynomial source classes, nonnegligible structure persists across the rapidly expanding catalogue of stable interactions, forcing logarithmic convergence. Exponential decay suppresses those high-order interactions fast enough to recover a power law, with different thresholds for sets and graphs because their invariant coordinates grow at different rates.

\subsection{Why classical in-domain power laws change}

In classical fixed-domain nonparametric learning, the number of features up to resolution \(k\) often grows polynomially, say \(q(k)\asymp k^s\). Polynomial source decay \(B(k)\asymp k^{-\tau}\) then produces the familiar power law \(n^{-\tau/(\tau+s)}\). Size generalization on sets and graphs is different. Uniformity over every larger size exposes all invariant interactions, whose stable counts grow as
\[
\exp(C_r k^{r/(r+1)})
\quad\text{and}\quad
\exp(k\log k),
\]
respectively. Consequently, the same polynomial source decay yields only logarithmic convergence in \(n\). A power law is recovered only when the interaction tail has sufficiently fast exponential source decay. This is a statistical effect: it remains even when transfer is exact and optimization is ideal.

To determine whether these rates reflect a particular estimator or an inherent limit of the problem, let \(\mathfrak F_{\mathsf X}(B)\) contain all size-consistent targets of type \(\mathsf X\in\{\mathrm{set},\mathrm{graph}\}\) with source envelope \(B\). We compare against every measurable procedure, including procedures that do not use polynomial truncation and are not themselves invariant. Its worst-case risk is
\[
\mathfrak R_n(\mathsf X,B)
\coloneqq
\inf_{\hat{\boldsymbol f}}\sup_{d\ge1}
\sup_{\boldsymbol f^\star\in\mathfrak F_{\mathsf X}(B)}
\mathbb E\!\left[\sup_{D\ge d}\|\hat f_D-f_D^\star\|_D^2\right].
\]
For the lower bound, the source class must genuinely contain the invariant directions counted by representation stability. Concretely, at every order \(k\), we require a stable-size regression or density submodel with \(q_{\mathsf X}(k)\) orthogonal local perturbations whose neighboring data distributions have quadratic KL divergence. Gaussian-noise regression and bounded local density or exponential-family perturbations satisfy this condition; Appendix~\ref{app:proof-minimax-size-generalization} states and verifies the precise construction.

\begin{theorem}[Minimax optimality]
\label{thm:minimax-size-generalization}
Under the preceding local richness condition, every rate in Theorem~\ref{thm:size-generalization-oracle} is minimax optimal up to constants and the displayed lower-order factors. In particular, for polynomial source decay,
\[
\mathfrak R_n(\mathrm{set},k^{-\tau})
\asymp(\log n)^{-\tau(r+1)/r},
\qquad
\mathfrak R_n(\mathrm{graph},k^{-\tau})
\asymp\left(\frac{\log n}{\log\log n}\right)^{-\tau}.
\]
For exponential source decay, the minimax risks have the same supercritical, critical, and subcritical regimes stated in Theorem~\ref{thm:size-generalization-oracle}. These lower bounds hold already at a suitable stable training size and therefore continue to hold when the loss takes a supremum over all \(D\ge d\).
\end{theorem}

The theorem rules out improvements by a different architecture, regularizer, optimization method, or transfer rule: the infimum in \(\mathfrak R_n\) ranges over all estimators. It also explains the role of exact transfer. Exact transfer removes cross-size mismatch, but it cannot create information about the many statistically distinguishable invariant directions that remain uncertain from \(n\) observations. The logarithmic regime is therefore not an artifact of polynomial regression or of the proof technique; it is the optimal statistical behavior of the entire source class.

\begin{corollary}[Exponential separation in sample complexity]
\label{cor:sample-complexity-separation}
Fix the confidence level and suppose \(B(k)\asymp k^{-\tau}\). Up to constants and lower-order factors, achieving squared error at most \(\varepsilon\) at a fixed size requires
\[
n_{\mathrm{in}}(\varepsilon)
=
\begin{cases}
\varepsilon^{-(\tau+rd)/\tau}, & \text{sets},\\
\varepsilon^{-(\tau+rd^2)/\tau}, & \text{graphs},
\end{cases}
\]
whereas uniform size generalization requires
\[
n_{\mathrm{size}}(\varepsilon)
=
\begin{cases}
\displaystyle
\exp\!\left(\Theta\!\left(\varepsilon^{-\frac{r}{\tau(r+1)}}\right)\right),
& \text{sets},\\[5pt]
\displaystyle
\exp\!\left(\Theta\!\left(\varepsilon^{-1/\tau}\log(1/\varepsilon)\right)\right),
& \text{graphs}.
\end{cases}
\]
Thus the same polynomial source class has polynomial sample complexity in-domain but exponential sample complexity when one estimator must transfer uniformly to all larger sizes.
\end{corollary}

\subsection{Discussion}

Our results are stated at the level of function classes rather than a particular neural parameterization. Consequently, they are compatible with permutation-equivariant graph architectures, including graph Transformers, when a permutation-invariant readout produces a graph-level prediction. Such an architecture can parameterize an estimator in our graph setting, but its symmetry alone does not imply the rates above. The cross-size target compatibility, exact norm-controlled transfer, stable feature growth, and source-decay conditions must still be verified. We therefore do not claim that every graph Transformer size-generalizes at these rates; rather, the theorems characterize the statistical behavior of those invariant graph-learning problems that satisfy the stated conditions.

This distinction separates three properties that are often conflated. First, an architecture may be \emph{applicable} to objects of different sizes. Second, a fixed predictor may be \emph{transferable}, meaning that its outputs remain stable as size changes. Third, a predictor learned from finite data may achieve \emph{statistical size generalization}, meaning that its error at unseen sizes decreases at a controlled rate as the number of training samples grows. The first two properties make cross-size prediction possible, but neither determines its sample complexity. Our results address the third property.

The minimax bounds also clarify what would be required to improve the rates. In the polynomial source class, no alternative architecture, regularizer, optimizer, or transfer rule can uniformly replace the logarithmic rate by a power law. Faster rates are possible only by imposing additional structure that reduces the effective class, such as faster source decay or a smaller family of transferable interactions. Thus the exponential separation is not a statement that all individual tasks are hard; it states that polynomial source decay alone is too weak to guarantee efficient size generalization uniformly over the class.

Finally, exact transfer isolates the statistical difficulty studied here by removing cross-size approximation mismatch. With approximate transfer, an additional error term would remain even with unlimited data. Likewise, problems with different symmetries require their own stable feature counts and need not exhibit the same thresholds. These are natural directions for extending the present theory, but they are distinct from the finite-sample phenomenon characterized by our results.

Altogether, the main message of this paper is:
\begin{tcolorbox}
Exact transfer enables size generalization, but target regularity determines its statistical efficiency. Fast exponential source decay yields nearly parametric rates, whereas polynomial decay yields minimax-optimal logarithmic rates.
\end{tcolorbox}

\section{Conclusion}
\label{sec:conclusion}

We developed a statistical theory of size generalization for supervised regression and unsupervised density estimation on sets and graphs. By combining exact transfer, representation stability, and source decay in invariant polynomial coordinates, we obtained rates that are uniform over both the training and deployment sizes. Polynomial source decay leads to minimax-optimal logarithmic rates, producing an exponential separation from classical in-domain power laws. Sufficiently fast exponential source decay overcomes invariant feature growth and recovers nearly parametric size generalization, with sharp thresholds that differ between sets and graphs.
These results show that architectural transferability alone does not determine the statistical efficiency of size generalization. The decisive interaction is between the growth of transferable invariant representations and the decay of high-order structure in the target. The matching minimax bounds establish that this phase transition is intrinsic rather than an artifact of the estimator or analysis.

\subsection*{AI use statement}

Large language model tools (ChatGPT 5.6 Sol and Claude Fable 5.1) were used primarily for copyediting, including improvements to grammar, wording, clarity, and \LaTeX{} presentation. The
research direction, core ideas, theoretical results, and experimental design
are original contributions of the authors. The authors carefully reviewed and
verified all technical content and take full responsibility for the final
manuscript.

\bibliography{ref}
\bibliographystyle{iclr2027_conference}

\clearpage
\appendix
\setcounter{tocdepth}{2}
\addtocontents{toc}{\protect\setcounter{tocdepth}{2}}
\renewcommand{\contentsname}{Appendix Contents}
\tableofcontents
\clearpage

\section{Additional Related Work}

\paragraph{Learning under invariances.}
Symmetry and invariance are central principles in geometric machine learning. Deep Sets provide a classical architecture for sets \citep{zaheer2017deep}, while graph neural networks and higher-order invariant and equivariant architectures are widely used for graphs \citep{xu2018powerful,maron2018invariant,satorras2021n}. Probabilistic symmetry provides a distributional foundation for invariant and equivariant models \citep{bloem2020probabilistic}, while complementary work studies limitations of set representations \citep{wagstaff2019limitations} and the ability of graph networks to count substructures \citep{tahmasebi2023power}. Beyond architectures that hard-code invariance, architecture-agnostic approaches include canonicalization \citep{ma2024canonicalization,kaba2023equivariance,dym2024equivariant,shumaylovlie,tahmasebi2025regularity} and frame averaging \citep{punyframe,atzmon2022frame,lin2024equivariance}. Although symmetry must be discovered in some settings, either adaptively \citep{tahmasebi2026adaptive} or via statistical tests \citep{soleymani2025robust,soleymani2026unified}, here we focus on the case of known symmetries. Applications of learning under invariances span reinforcement learning \citep{wang2022mathrm}, physics \citep{bogatskiy2020lorentz,nguyen2024theory}, drug discovery \citep{bongini2021molecular}, and social networks \citep{li2023survey}. Our work is complementary to these modeling approaches: rather than proposing a new architecture, we study when invariant statistical estimation can transfer across dimensions.

\paragraph{Methods and limitations for size generalization.}
Beyond the transferability theory discussed in the main text, several works develop model-specific approaches to graph size generalization through Wasserstein barycenter matching, size-shift regularization, or disentangled representations \citep{chu2023wasserstein,buffelli2022sizeshiftreg,huang2024enhancing}. Graph metrics have also been used to relate structural shifts to GNN stability \citep{chuang2022tree}. Other results identify architectural limitations to sharp size generalization \citep{velivckovic2024softmax} or derive graph-size-dependent statistical guarantees from a manifold perspective \citep{wang2024manifold}. These works clarify mechanisms that help or hinder generalization across sizes. Our focus is instead on size-uniform sample complexity and matching minimax rates for invariant estimation.

\paragraph{Generalization and sample complexity under symmetry.}
The statistical benefits of symmetry have been widely studied in learning theory. Prior work has shown that invariance can reduce sample complexity in supervised learning and related estimation problems \citep{tahmasebi2023exact,elesedy2021provably,wang2020incorporating,petrache2023approximation,tahmasebigeneralization}, with extensions to estimating probability divergences \citep{chen2023sample,tahmasebi2023sample}. More broadly, the study of symmetry and scaling laws extends to comparisons between approximate and exact symmetry \citep{tahmasebi2026achieving}, data augmentation via Fourier analysis on groups \citep{tahmasebi2026data,tahmasebi2026sparse}, and statistical-computational trade-offs \citep{soleymanilearning,soleymani2026efficient}. Generalization bounds for graph neural networks have also been developed using Rademacher complexity \citep{garg2020generalization}, VC dimension \citep{scarselli2018vapnik,morris2023wl}, and PAC-Bayesian analysis \citep{liao2020pac,ju2023generalization}. Further work studies GNN generalization and transfer \citep{maskey2022generalization,maskey2025generalization,levie2023graphon,vasileiou2024covered,rauchwerger2024generalization}; see also the survey of \citet{vasileiou2025survey}. Our contribution differs by focusing on invariant estimation across changing dimensions, with explicit dimension-uniform rates and phase transitions driven by stable feature complexity.

\paragraph{Source conditions and optimal rates.}
Classical analyses of regularized least squares derive optimal statistical rates under source and capacity conditions \citep{caponnetto2007optimal}. More recent work establishes exact and minimax rates for kernel ridge regression under source conditions in large dimensions \citep{zhang2025optimal}. These results concern fixed domains or joint asymptotics of sample size and ambient dimension. We instead study transfer from one object size to all larger sizes and determine when the resulting rates can be uniform over both the training and deployment sizes.


\section{Function-Level Symmetry, Statistical Learning, and Transfer}
\label{app:preliminaries-invariant-polynomials}

This appendix formalizes the function-level framework of Section~\ref{sec:problem-statement}. It explains why unlabeled set and graph laws yield invariant density and regression targets, records representative losses under the unified observation notation, and defines exact transfer through coherent linear isometries. It then develops the stable invariant function spaces used in the rate calculations.

\subsection{Group actions on data and functions}
\label{app:group-actions}

A group \(G\) acts on a measurable space \(\mathcal X\) if
\[
e\cdot x=x,
\qquad
g\cdot(h\cdot x)=(gh)\cdot x,
\]
for every \(g,h\in G\) and \(x\in\mathcal X\). This action induces a linear action on real-valued functions \(f:\mathcal X\to\mathbb R\) by pullback:
\[
(g\cdot f)(x)\coloneqq f(g^{-1}\cdot x).
\]
A function is \(G\)-invariant when \(g\cdot f=f\) for every \(g\in G\), equivalently,
\[
f(g\cdot x)=f(x)
\qquad
\text{for every }g\in G,\ x\in\mathcal X.
\]
For sets and graphs of size \(d\), the group is \(G=S_d\), acting by relabeling elements or nodes as specified in Section~\ref{sec:problem-statement}.

A probability distribution \(\mu\) on \(\mathcal{X}\) is \(G\)-invariant if
\[
g\cdot X\overset{\mathrm{law}}{=}X
\qquad
\text{for every }g\in G .
\]
Equivalently, \(\mu(g^{-1}A)=\mu(A)\) for all measurable sets \(A\subseteq\mathcal{X}\).

\paragraph{Density targets.}
Suppose \(\mu\) admits a density \(f^\star\) with respect to a \(G\)-invariant reference measure \(\lambda\). For every measurable \(A\),
\[
\int_A f^\star\,d\lambda
=\mu(A)
=\mu(g^{-1}A)
=\int_A f^\star(g^{-1}\cdot x)\,d\lambda(x).
\]
Hence \(f^\star(g\cdot x)=f^\star(x)\) for \(\lambda\)-almost every \(x\). Thus an unlabeled set or graph distribution has an invariant density whenever its density is expressed relative to an invariant reference measure.

\paragraph{Regression targets.}
Let \(\rho\) be a joint distribution on \(\mathcal X\times\mathbb R\) satisfying
\[
(g\cdot X,Y)\overset{\mathrm{law}}{=}(X,Y).
\]
Then the conditional law of \(Y\) given \(X=x\) agrees with that given \(X=g\cdot x\), up to the usual almost-sure choice of a conditional version. Consequently, the regression function
\[
f^\star(x)=\mathbb E[Y\mid X=x]
\]
admits a \(G\)-invariant version.

\paragraph{Invariant risk minimization.}
Let \(\mathcal F\) be a function class closed under the induced \(G\)-action. A population risk \(L:\mathcal F\to\mathbb R\) is \(G\)-invariant if \(L(g\cdot f)=L(f)\). If the minimizer is unique, it is necessarily invariant. If the minimizer is not unique, \(\mathcal F\) is convex, and \(L\) is convex, then group averaging produces an invariant minimizer:
\[
\overline f(x)
\coloneqq
\int_G (g\cdot f)(x)\,d\eta_G(g),
\]
where \(\eta_G\) is the uniform measure for finite groups and Haar probability measure for compact groups. Indeed, Jensen's inequality and invariance give
\[
L(\overline f)
\le
\int_G L(g\cdot f)\,d\eta_G(g)
=L(f).
\]
This argument takes place entirely in function space and requires no action on a parameterization. Accordingly, uniqueness means uniqueness as an \(L^2\) function, up to almost-everywhere equality; it is compatible with many parameter values representing that same function.

\subsection{Examples of statistical learning objectives}
\label{app:m-estimation-examples}

Recall the unified observation \(Z_i=X_i\) for unsupervised learning and \(Z_i=(X_i,Y_i)\) for supervised learning. In either case, \(\widehat L_{n,d}(f)=n^{-1}\sum_{i=1}^n\ell_d(Z_i,f)\) estimates \(L_d(f)=\mathbb E[\ell_d(Z,f)]\). The formulation includes the following standard tasks.

\paragraph{Unsupervised statistical learning: density estimation.}
Let \(\mathcal F_d\subseteq L^2(\lambda_d)\) be a class of nonnegative functions satisfying \(\int f\,\mathrm d\lambda_d=1\). Since \(\mu_d=f_d^\star\lambda_d\), the negative log-likelihood loss
\[
\ell_d(x,f)=-\log f(x)
\]
has population risk
\[
L_d(f)=\mathbb E_{X\sim\mu_d}[-\log f(X)].
\]
This risk is uniquely minimized by \(f_d^\star\), up to \(\lambda_d\)-almost-everywhere equality, whenever \(f_d^\star\in\mathcal F_d\) and the risk is finite. Restricting to \(\mathcal F_d^{\mathrm{inv}}\) gives invariant maximum likelihood directly at the function level.

\paragraph{Supervised statistical learning: regression.}
For real-valued labels, squared loss gives
\[
\ell_d(x,y,f)=(y-f(x))^2,
\qquad
f_d^\star(x)=\mathbb E[Y\mid X=x].
\]
More generally, any proper supervised loss may be used when its population minimizer is the desired conditional target. If relabeling \(X\) does not change the joint law of \((X,Y)\), the target is invariant.

\paragraph{Score matching and denoising.}
The scalar formulation is not tied to likelihood: proper scoring rules, density-ratio objectives, and variational losses also apply whenever their population minimizer identifies \(f_d^\star\). Score matching and denoising instead estimate a vector-valued field, such as \(\nabla\log f_d^\star\), in \(L^2(\lambda_d;\mathbb R^{m_d})\), where \(m_d=rd\) for sets and \(m_d=rd^2\) for graphs. Invariance of the density makes its score permutation equivariant. Replacing invariant scalar spaces by the corresponding equivariant vector spaces leaves the transfer, truncation, and estimation arguments unchanged, provided the same representation-stability and source-decay conditions hold.

\subsection{Exact functional transfer across sizes}
\label{app:any-dimensional-models}

For density estimation let \(\mathcal H_d=L^2(\lambda_d)\), and for regression let \(\mathcal H_d=L^2(\mu_d)\). In either case, let \(\mathcal H_d^{\mathrm{inv}}\) be its invariant subspace, with inner product \(\langle\cdot,\cdot\rangle_d\) and norm \(\|\cdot\|_d\). An exact transfer system is a collection of linear maps
\[
T_{d\to D}:\mathcal H_d^{\mathrm{inv}}\longrightarrow
\mathcal H_D^{\mathrm{inv}},
\qquad d\le D,
\]
satisfying the following three properties:
\begin{enumerate}[label=\textup{(\roman*)},leftmargin=2.2em]
    \item \emph{coherence:}
    \(T_{d\to d}=\mathrm{Id}\) and
    \(T_{D\to E}T_{d\to D}=T_{d\to E}\);
    \item \emph{isometry:}
    \(\langle T_{d\to D}f,T_{d\to D}h\rangle_D
    =\langle f,h\rangle_d\);
    \item \emph{target compatibility:}
    \(T_{d\to D}f_d^\star=f_D^\star\).
\end{enumerate}
The isometry identity implies
\[
\|T_{d\to D}f-T_{d\to D}h\|_D
=\|f-h\|_d.
\]
The map need not reach every function at the larger size; it only identifies the size-stable subspace carrying the shared rule.

These assumptions express the structure needed to learn at one size and make a statistical claim at another. Without target compatibility, the unseen target \(f_D^\star\) is not identifiable from size-\(d\) observations. Without norm control, transfer could arbitrarily amplify a small source-size error. Isometry rules out that artifact and makes errors comparable across sizes. For score matching or denoising, the same definition applies to the corresponding equivariant vector-valued spaces.

\paragraph{Stable invariant features.}
Suppose \(\{\psi_{d,j}\}_{j\ge1}\) and \(\{\psi_{D,j}\}_{j\ge1}\) are corresponding orthonormal invariant features. Define
\[
T_{d\to D}\psi_{d,j}=\psi_{D,j}
\]
and extend linearly and continuously. Then, for every square-summable sequence \((a_j)_j\),
\[
T_{d\to D}\!\left(\sum_j a_j\psi_{d,j}\right)
=\sum_j a_j\psi_{D,j},
\qquad
\left\|T_{d\to D}\sum_j a_j\psi_{d,j}\right\|_D^2
=\sum_j a_j^2.
\]
Thus transfer preserves the learned rule and its error exactly. This construction underlies the stable polynomial and Hermite classes below. It is also the function-space abstraction of weight sharing in Deep Sets and GNNs: the computation is instantiated on a new number of elements or nodes, but the rule itself is unchanged \citep{zaheer2017deep,xu2018powerful,levin2024any}.

\paragraph{Common ambient objects.}
Exact transfer also arises when all finite-size functions are realizations of a single rule on a common limit object. Mean embeddings and normalized empirical measures provide such objects for sets. Graphons and normalized motif functionals provide them for dense graphs \citep{ruiz2020graphon,diaconis2008graphlimits}. Appendix~\ref{app:sets-graphs-examples} gives explicit set and graph constructions, including replication, graph blow-ups, and shared Deep Sets and GNN rules.

\subsection{Invariant polynomial function spaces}
\label{app:invariant-polynomial-parameters}

A central example is a finite-dimensional space of polynomial functions. Let \(V_d\) be a finite-dimensional vector space on which \(G_d\) acts linearly, and let \(\mathcal P_{d,k}\) be the space of homogeneous degree-\(k\) polynomial functions on \(V_d\). The induced action on functions is
\[
(g\cdot f)(x)\coloneqq f(g^{-1}\cdot x).
\]
The invariant homogeneous polynomial space is
\[
\mathcal{P}_{d,k}^{G_d}
\coloneqq
\{f\in\mathcal{P}_{d,k}:g\cdot f=f \text{ for all }g\in G_d\}.
\]
For density estimation, completeness follows from the Hermite basis of \(L^2(\lambda_d)\). For regression, we assume the standard polynomial-completeness condition in \(L^2(\mu_d)\), which holds, for example, under suitable moment determinacy or exponential-moment conditions. Averaging polynomial approximants over \(S_d\) then proves density of invariant polynomials in \(\mathcal H_d^{\mathrm{inv}}\).

For finite groups, the dimension of the invariant subspace can be expressed using characters. If \(\varrho_{d,k}\) is the representation of \(G_d\) on \(\mathcal{P}_{d,k}\) (i.e., the collection of matrices corresponding to the action) and \(\chi_{d,k}(g)=\mathrm{tr}(\varrho_{d,k}(g))\) is its character, then
\[
\dim \mathcal{P}_{d,k}^{G_d}
=
\frac{1}{|G_d|}\sum_{g\in G_d}\chi_{d,k}(g).
\]
Equivalently, if \(P_{G_d}\) denotes the group-averaging projection
\[
P_{G_d}f
=
\frac{1}{|G_d|}\sum_{g\in G_d}g\cdot f,
\]
then
\[
\dim \mathcal{P}_{d,k}^{G_d}
=
\mathrm{tr}(P_{G_d}).
\]
For compact groups, the finite average is replaced by Haar integration.

Representation stability is the phenomenon that these invariant function-space dimensions stabilize as \(d\) grows. This stabilization is the algebraic mechanism behind dimension-uniform learning in our setting.

\subsection{Set-valued data and multisymmetric polynomials}
\label{app:sets-multisymmetric-polynomials}

Consider set-valued data with \(d\) elements in \(\mathbb{R}^r\). We represent the set as
\[
X=(x_1,\ldots,x_d)\in(\mathbb{R}^r)^d,
\qquad
x_i=(x_{i,1},\ldots,x_{i,r})\in\mathbb{R}^r .
\]
The symmetric group \(S_d\) acts by permuting the \(d\) elements:
\[
\pi\cdot X=(x_{\pi^{-1}(1)},\ldots,x_{\pi^{-1}(d)}).
\]
The corresponding invariant polynomials are called multisymmetric polynomials: they are polynomials in the \(dr\) variables \(x_{i,a}\) that are invariant under permutations of the block index \(i\).

Let \(\mathcal{P}_{d,r,k}^{S_d}\) denote the space of homogeneous degree \(k\) multisymmetric polynomials. A monomial is described by a collection of multi-indices
\[
\alpha_i=(\alpha_{i,1},\ldots,\alpha_{i,r})\in\mathbb{N}^r,
\qquad
1\le i\le d,
\]
with total degree \(\sum_i|\alpha_i|=k\). The group \(S_d\) permutes the multi-indices \(\alpha_i\). Therefore an invariant monomial orbit is determined by a multiset of nonzero multi-indices
\[
\{\alpha^{(1)},\ldots,\alpha^{(q)}\},
\qquad
\alpha^{(j)}\in\mathbb{N}^r\setminus\{0\},
\qquad
\sum_{j=1}^q|\alpha^{(j)}|=k.
\]
For \(d\ge k\), every such multiset can be realized using at most \(k\) nonzero blocks, so the dimension stabilizes in \(d\).

For each \(m\ge1\), the number of nonzero block types of degree \(m\) is
\[
a_r(m)=\#\{\alpha\in\mathbb{N}^r:|\alpha|=m\}
=
\binom{r+m-1}{m}.
\]
Hence the stable dimension of homogeneous degree \(k\) invariants is
\[
N_r(k)
\coloneqq
\dim \mathcal{P}_{d,r,k}^{S_d}
=
[t^k]\prod_{m\ge1}(1-t^m)^{-a_r(m)}
\qquad
\text{for } d\ge k,
\]
where \([t^k]F(t)\) denotes the coefficient of \(t^k\) in the formal power series \(F(t)\). Equivalently,
\[
N_r(k)
=
[t^k]\prod_{m\ge1}(1-t^m)^{-\binom{r+m-1}{m}}.
\]
For \(r=1\), this reduces to the ordinary partition number:
\[
N_1(k)=p(k).
\]
Thus the invariant degree-\(k\) dimension is independent of \(d\) once \(d\ge k\), whereas the full homogeneous degree-\(k\) polynomial space in \(dr\) variables has dimension
\[
\binom{dr+k-1}{k},
\]
which grows polynomially in \(d\) of order \(d^k\) for fixed \(k\).

The asymptotic behavior of \(N_r(k)\) for fixed \(r\) and large \(k\) is subexponential:
\[
\log N_r(k)
\sim
C_r k^{\frac{r}{r+1}},
\]
where
\[
C_r
=
\frac{r+1}{r}
\big(r\zeta(r+1)\big)^{\frac{1}{r+1}},
\]
and \(\zeta(s)\coloneqq \sum_{n=1}^\infty n^{-s}\) denotes the Riemann zeta function. Thus
\[
N_r(k)
=
\exp\left(
C_r k^{\frac{r}{r+1}}
+
o\left(k^{\frac{r}{r+1}}\right)
\right).
\]
For example, \(r=1\) recovers the Hardy--Ramanujan exponent
\[
\log p(k)\sim \pi\sqrt{\frac{2k}{3}}.
\]
For \(r=2\),
\[
\log N_2(k)\sim \frac{3}{2}(2\zeta(3))^{1/3}k^{2/3},
\]
and for \(r=3\),
\[
\log N_3(k)\sim \frac{4}{3}(3\zeta(4))^{1/4}k^{3/4}.
\]

These formulas show that for fixed element dimension \(r\) and polynomial degree \(k\), the invariant polynomial function-space dimension stabilizes as the set size \(d\) grows. This is the representation-stability mechanism underlying the dimension-independent rates for set-valued data.

\subsection{Hermite features for set-valued data}
\label{app:hermite-set-features}

For data on \((\mathbb{R}^r)^d\), a natural basis is given by multivariate Hermite polynomials relative to a Gaussian reference measure. Let \(H_\alpha\) denote an orthonormal Hermite polynomial indexed by \(\alpha\in\mathbb{N}^{dr}\). For block-structured data, we may write
\[
H_{\alpha_1,\ldots,\alpha_d}(X)
=
\prod_{i=1}^d H_{\alpha_i}(x_i),
\qquad
\alpha_i\in\mathbb{N}^r.
\]
The group \(S_d\) permutes the block indices \(\alpha_i\). Therefore invariant Hermite features are obtained by symmetrizing Hermite monomials over \(S_d\). The dimension count for homogeneous degree \(k\) invariant Hermite features is the same as for multisymmetric polynomials:
\[
\dim \mathcal{H}_{d,r,k}^{S_d}=N_r(k)
\qquad
\text{for } d\ge k.
\]

This construction gives a concrete hierarchy for invariant density and regression functions. For density estimation, let \(\phi_{d,j}\) be invariant Hermite functions; for regression, let them be the corresponding invariant polynomials orthonormalized in \(L^2(\mu_d)\). In either case, the scalar target admits an expansion
\[
f_d^\star(X)
=
\sum_j a_{d,j}^\star \phi_{d,j}(X).
\]
Truncation at degree \(k\) produces a finite-dimensional invariant function class whose dimension is
\[
N_r(\le k)=\sum_{q=0}^k N_r(q),
\]
independent of \(d\) once \(d\ge k\). For density estimation, the admissible class additionally enforces nonnegativity and unit integral.

If one instead studies score matching, an invariant log-density expansion \(\log f_d^\star=\sum_j b_{d,j}^\star\phi_{d,j}\) gives the Lebesgue score
\[
\nabla_X \log\frac{\mathrm d\mu_d}{\mathrm d x}(X)
=
-X+\sum_j b_{d,j}^\star \nabla_X\phi_{d,j}(X).
\]
Gradients of invariant scalar features are equivariant vector fields:
\[
\nabla_X \phi_{d,j}(\pi\cdot X)=\pi\cdot\nabla_X\phi_{d,j}(X).
\]
If the function coefficients have a controlled tail, for example,
\[
\sum_{j>m}(a_{d,j}^\star)^2\le b_m^2,
\]
then truncating to \(m\) invariant Hermite functions incurs squared \(L^2\) approximation error at most \(b_m^2\). This is a concrete instance of the approximation tail \(B(k)\) in Theorem~\ref{thm:size-generalization-oracle}.

\subsection{Attributed graphs}
\label{app:attributed-graph-invariants}

An attributed graph with \(d\) nodes is represented, as in the main text, by
\[
X\in\mathcal X_d=\mathbb R^{r\times d\times d}.
\]
The channel index \(a\in\{1,\ldots,r\}\) records edge and/or node features; node features may, for example, be placed on diagonal entries of designated channels. Node relabeling acts simultaneously on both node indices:
\[
(\pi\cdot X)_{a,i,j}
=X_{a,\pi^{-1}(i),\pi^{-1}(j)}.
\]

A homogeneous degree-\(k\) polynomial function is a linear combination of monomials
\[
\prod_{a=1}^r\prod_{i,j=1}^d
X_{a,i,j}^{m_{a,i,j}},
\qquad
\sum_{a,i,j}m_{a,i,j}=k.
\]
Two such monomials belong to the same \(S_d\)-orbit precisely when their node indices differ only by a relabeling. The orbits are therefore indexed by unlabeled channel-attributed directed multigraph patterns of total degree \(k\), with diagonal factors representing vertex decorations when node features are encoded on the diagonal.

Let \(\mathcal R_{d,r,k}^{S_d}\) be the corresponding space of invariant homogeneous polynomial functions. Any degree-\(k\) monomial involves at most \(2k\) distinct node indices, so
\[
\dim\mathcal R_{d,r,k}^{S_d}
=G_r(k)
\qquad
\text{for every }d\ge2k,
\]
where \(G_r(k)\) is the number of admissible unlabeled channel-attributed multigraph patterns of degree \(k\). Thus the function-space dimension stabilizes with respect to the number of nodes.

The same dimension can be written by character theory. With
\[
W_d=\mathbb R^{r\times d\times d},
\]
we have
\[
\mathcal R_{d,r,k}^{S_d}
=\left(\operatorname{Sym}^k(W_d^\ast)\right)^{S_d},
\qquad
\dim\mathcal R_{d,r,k}^{S_d}
=\frac1{d!}\sum_{\pi\in S_d}
\chi_{\operatorname{Sym}^k(W_d^\ast)}(\pi).
\]
Equivalently, if \(\varrho_d\) is the representation of \(S_d\) on \(W_d^\ast\), its Molien series is
\[
\sum_{k\ge0}\dim\mathcal R_{d,r,k}^{S_d}t^k
=
\frac1{d!}\sum_{\pi\in S_d}
\frac1{\det(I-t\varrho_d(\pi))}.
\]

The asymptotic dependence on \(k\) separates node-only and edge-attributed systems. If all active channels are supported on the diagonal, the stable dimension reduces to the multisymmetric count
\[
N_r(k)
=
[t^k]\prod_{m\ge1}(1-t^m)^{-\binom{r+m-1}{m}},
\]
and therefore
\[
N_r(k)
=
\exp\left(
C_r k^{\frac{r}{r+1}}
+
o\left(k^{\frac{r}{r+1}}\right)
\right).
\]
In contrast, if at least one channel contains edge variables, graph-pattern complexity gives
\[
G_r(k)
=
\exp\!\left(k\log k + O(k\log\log k)\right),
\qquad k\to\infty .
\]
Thus attributed-graph invariant function spaces stabilize in the number of nodes \(d\), while their growth with interaction order \(k\) is larger than for set-valued data.

\section{Concrete Exact Transfers for Sets and Graphs}
\label{app:sets-graphs-examples}

This appendix makes the transfer assumption concrete for the two domains studied in the paper. The operator \(T_{d\to D}\) does not match arbitrary functions across dimensions; it instantiates the same invariant rule at a new set or graph size. Shared Deep Sets, GNNs, empirical measures, and graphons motivate coherence and exact target compatibility. The stable feature constructions additionally realize the linear isometry required by the theory after normalization in the relevant \(L^2\) spaces.

\subsection{Sets}

A set of \(d\) vectors is represented by \(X=(x_1,\ldots,x_d)\in\mathcal X_d=(\mathbb R^r)^d\simeq\mathbb R^{r\times d}\). Because the ordering carries no information, \(S_d\) acts by \(\pi\cdot X=(x_{\pi^{-1}(1)},\ldots,x_{\pi^{-1}(d)})\).
A set distribution is therefore represented by an exchangeable law, and both its density and its regression function are invariant by Appendix~\ref{app:group-actions}.

\paragraph{Shared Deep Sets rules.}
Let \(\phi:\mathbb R^r\to\mathbb R^m\) be an elementwise feature map and define the normalized set summary
\[
m_d^\phi(X)=\frac1d\sum_{i=1}^d\phi(x_i).
\]
For any readout \(F:\mathbb R^m\to\mathbb R\), the rule
\[
f_d(X)=F\!\left(m_d^\phi(X)\right)
\]
is invariant for every \(d\). Its transfer to size \(D\) is obtained by applying the same \(\phi\) and \(F\):
\[
(T_{d\to D}f_d)(X_1,\ldots,X_D)
=F\!\left(\frac1D\sum_{i=1}^D\phi(X_i)\right).
\]
This is precisely the size-shared computation of a mean-pooled Deep Sets model. Sum-pooled models admit the same construction when the task uses an extensive, rather than normalized, quantity. Exact target compatibility means that the population task is generated by this same rule at every size.

\paragraph{Replication and empirical measures.}
If \(D=qd\), repeat each element \(q\) times:
\[
R_{d\to D}(x_1,\ldots,x_d)
=
(\underbrace{x_1,\ldots,x_1}_{q},\ldots,
\underbrace{x_d,\ldots,x_d}_{q}).
\]
The normalized summary is unchanged,
\[
m_D^\phi(R_{d\to D}X)=m_d^\phi(X).
\]
Equivalently, replication leaves the empirical measure \(d^{-1}\sum_i\delta_{x_i}\) unchanged. Hence every function of the empirical measure, including normalized moments and mean-pooled Deep Sets rules, transfers exactly. Zero padding gives another exact construction for sum-based rules whenever the added null element satisfies \(\phi(0)=0\).

\paragraph{Stable feature expansions.}
Let \(\psi_{d,j}\) be normalized multisymmetric or invariant Hermite functions, paired across sizes by their common orbit type. For
\[
f_d=\sum_j a_j\psi_{d,j},
\]
define
\[
T_{d\to D}f_d=\sum_j a_j\psi_{D,j}.
\]
If the paired features are orthonormal in their respective task spaces, Parseval's identity gives
\[
\|T_{d\to D}f_d\|_D^2
=\sum_j a_j^2
=\|f_d\|_d^2.
\]
Thus \(T_{d\to D}\) is a linear isometry. A shared expansion may represent either a regression function or a density; in the latter case, the admissible coefficient class is restricted so that the expansion remains nonnegative and integrates to one. This provides the explicit transfer used by the set function classes in Appendix~\ref{app:sets-multisymmetric-polynomials}.

\subsection{Attributed graphs}

An attributed graph with \(d\) nodes is represented by \(X\in\mathcal X_d=\mathbb R^{r\times d\times d}\), and \(S_d\) acts by simultaneous relabeling:
\[
(\pi\cdot X)_{a,i,j}
=X_{a,\pi^{-1}(i),\pi^{-1}(j)}.
\]
Consequently, graph-level densities and regression functions depend only on the isomorphism class of the attributed graph.

\paragraph{Normalized motif rules and graph blow-ups.}
For a fixed attributed pattern \(H\), let \(t_H(X)\) denote its normalized homomorphism density in \(X\). These functions are invariant and are indexed by \(H\), not by the number of nodes. If \(X^{(q)}\) is the \(q\)-fold blow-up obtained by replacing each node by \(q\) twins and each edge by the corresponding complete bipartite block, then
\[
t_H(X^{(q)})=t_H(X).
\]
Thus graph blow-ups preserve every rule of the form
\[
f_d(X)=F\bigl(t_{H_1}(X),\ldots,t_{H_m}(X)\bigr)
\]
exactly. A step-function graphon gives the same interpretation: refining its finite partition changes the graph size but not the common graphon or its normalized motif functionals.

\paragraph{Shared GNN rules.}
A message-passing GNN applies the same local aggregation and update functions at every node and every size, followed by a permutation-invariant graph readout. Holding these functions fixed while changing \(d\) therefore defines a map from the size-\(d\) function to its size-\(D\) realization. This is the architecture-level instance of \(T_{d\to D}\). The transfer is exact when the target family is generated by the same shared GNN rule. Mean aggregation and normalized graph readouts prevent the numerical scale of the rule from growing merely with the number of nodes.

\paragraph{Stable graph feature expansions.}
Let \(\psi_{d,H}\) be normalized invariant functions indexed by unlabeled attributed multigraph patterns \(H\). Transfer pairs the same pattern across sizes:
\[
T_{d\to D}\psi_{d,H}=\psi_{D,H}.
\]
Extending this identification linearly gives
\[
T_{d\to D}\!\left(\sum_H a_H\psi_{d,H}\right)
=\sum_H a_H\psi_{D,H}.
\]
After orthonormalization in the task norms, this map is an isometry and satisfies the coherence identity. The degree-\(k\) truncations stabilize once \(d\ge2k\), as shown in Appendix~\ref{app:attributed-graph-invariants}. These expansions cover invariant graph regression and density functions without privileging a particular learning loss; for densities, admissibility again enforces nonnegativity and normalization.

\subsection{What the exact-transfer assumption asserts}

Across all of these examples, exact transfer makes two statements. First, the semantic rule is shared across sizes: the same normalized summary, architecture, motif functional, or invariant expansion is used at the training and target sizes. Second, the task family is compatible with that rule:
\[
T_{d\to D}f_d^\star=f_D^\star.
\]
This does not claim that arbitrary set or graph tasks are related. It identifies the broad class of size-consistent problems for which learning at one size can be statistically informative at another. The isometry assumption then ensures that this semantic identification does not alter the scale of estimation error. Together, these conditions isolate the statistical question studied in the paper from a separate cross-size modeling mismatch.

\section{Proof of Theorem~\ref{thm:size-generalization-oracle}}
\label{app:proof-scaling-laws}

This section proves the dimension-free upper bounds and derives all rate regimes. The proof has two steps. First, exact isometric transfer converts a fixed-size estimation guarantee into a simultaneous guarantee at every larger size. Second, representation growth is balanced against the source tail.

Exact transfer is defined in Section~\ref{sec:problem-statement} and formalized in Appendix~\ref{app:any-dimensional-models}; representation growth \(q_{\mathsf X}(k)\) and the source envelope \(B(k)\) are defined in Section~\ref{sec:main-results}. The loss assumptions in the main theorem are used below to derive the finite-dimensional estimation bound. The minimax theorem additionally uses the local nondegeneracy conditions~\eqref{eq:local-hypercube}--\eqref{eq:local-kl}.

\subsection{Finite-dimensional estimation and transfer}

Fix a domain type \(\mathsf X\in\{\mathrm{set},\mathrm{graph}\}\), a training size \(d\), and an interaction order \(k\). Let \(f_{d,k}^{\circ}\) and \(\hat f_{n,d,k}^{\mathrm{inv}}\) minimize the population and empirical risks over \(\mathcal F_{d,k}\), respectively. Uniform calibration gives
\(
\|f_{d,k}^{\circ}-f_d^\star\|_d^2\le C_{\mathrm{app}}B_d(k).
\)
Apply the strong-convexity and score-concentration calculation of Appendix~\ref{app:proof-invariant-m-estimation} in orthonormal coordinates on \(\mathcal V_{d,k}\). Since those coordinates identify Euclidean distance with \(\|\cdot\|_d\), with probability at least \(1-\delta\),
\[
\bigl\|\hat f_{n,d,k}^{\mathrm{inv}}-f_{d,k}^{\circ}\bigr\|_d^2
\le
C\frac{\sigma^2}{\alpha^2}
\frac{\dim(\mathcal V_{d,k})+\log(1/\delta)}{n}.
\]
The squared triangle inequality and the calibration bound therefore give
\[
\bigl\|\hat f_{n,d,k}^{\mathrm{inv}}-f_d^\star\bigr\|_d^2
\le
C_0\left(
B_d(k)
+\frac{\dim(\mathcal V_{d,k})+\log(1/\delta)}{n}
\right).
\tag{D.1}\label{eq:finite-oracle}
\]
Here \(C_0\) depends on the uniform calibration, curvature, and score constants, but not on \(n,d,D\), or \(k\). Thus~\eqref{eq:finite-oracle} is a consequence of the primitive loss conditions, rather than an assumed rate.

The two quantities on the right are uniform in size: by definition, \(B_d(k)\le B(k)\) and \(\dim(\mathcal V_{d,k})\le q_{\mathsf X}(k)\). Moreover, exact target transfer and isometry give, simultaneously for every \(D\ge d\),
\[
\begin{aligned}
\bigl\|T_{d\to D}\hat f_{n,d,k}^{\mathrm{inv}}-f_D^\star\bigr\|_D^2
&=
\bigl\|T_{d\to D}(\hat f_{n,d,k}^{\mathrm{inv}}-f_d^\star)\bigr\|_D^2\\
&=
\bigl\|\hat f_{n,d,k}^{\mathrm{inv}}-f_d^\star\bigr\|_d^2.
\end{aligned}
\tag{D.2}\label{eq:isometric-risk-identity}
\]
No union bound over \(D\) is needed: all target-size errors equal the same source-size error on the event in~\eqref{eq:finite-oracle}. Taking the supremum in~\eqref{eq:isometric-risk-identity} therefore yields
\[
\sup_{D\ge d}
\bigl\|T_{d\to D}\hat f_{n,d,k}^{\mathrm{inv}}-f_D^\star\bigr\|_D^2
\le
C_0\left(
B(k)+\frac{q_{\mathsf X}(k)+\log(1/\delta)}{n}
\right).
\tag{D.3}\label{eq:uniform-oracle}
\]
This is the key dimension-free oracle inequality. Object sizes \(d\) and \(D\) have disappeared; only interaction order, representation growth, and target regularity remain.

\subsection{Set rates}
\label{app:set-rate-calculations}

Write \(\beta=r/(r+1)\). The stable multisymmetric count satisfies
\[
\log q_{\mathrm{set}}(k)=C_r k^\beta+o(k^\beta),
\qquad
C_r=\frac{r+1}{r}\bigl(r\zeta(r+1)\bigr)^{1/(r+1)}.
\tag{D.4}\label{eq:set-growth-proof}
\]
We optimize the two leading terms in~\eqref{eq:uniform-oracle}; for a fixed confidence level, \(\log(1/\delta)/n\) is smaller than all rates below.

\paragraph{Polynomial source decay.}
If \(B(k)\lesssim k^{-\tau}\), balancing the two terms in~\eqref{eq:uniform-oracle} gives \(C_rk_n^\beta+\tau\log k_n=\log n+o(\log n)\), hence \(k_n=(\log n/C_r)^{1/\beta}(1+o(1))\) and
\[
\sup_{D\ge d}
\bigl\|T_{d\to D}\hat f_{n,d,k_n}^{\mathrm{inv}}-f_D^\star\bigr\|_D^2
\lesssim
(\log n)^{-\tau/\beta}
=(\log n)^{-\tau(r+1)/r}.
\tag{D.5}\label{eq:set-polynomial-upper}
\]

\paragraph{Exponential source decay.}
If \(B(k)\lesssim e^{-ak^\nu}\), the balancing equation is
\[
a k_n^\nu+C_r k_n^\beta+o(k_n^\beta)=\log n.
\tag{D.6}\label{eq:set-exp-balance}
\]
Solving it in the three regimes gives the compact summary
\[
B(k_n)+\frac{q_{\mathrm{set}}(k_n)}n
\asymp
\begin{cases}
n^{-1+o(1)}, & \nu>\beta,\\
n^{-a/(a+C_r)+o(1)}, & \nu=\beta,\\
\exp\!\left[-aC_r^{-\nu/\beta}(\log n)^{\nu/\beta}(1+o(1))\right], & 0<\nu<\beta.
\end{cases}
\tag{D.7}\label{eq:set-exp-rates}
\]
The first line is the fast set regime in Table~\ref{tab:main-minimax-rates}; the other two lines give the critical and subcritical interpolation.

\subsection{Graph rates}
\label{app:graph-rate-calculations}

For every fixed feature dimension \(r\), the stable graph-pattern count satisfies
\[
\log q_{\mathrm{graph}}(k)
=k\log k+O_r(k\log\log k)
=k\log k\,(1+o(1)).
\tag{D.8}\label{eq:graph-growth-proof}
\]
Thus \(r\) affects only the lower-order remainder and does not enter the first-order rates.

\paragraph{Polynomial source decay.}
For \(B(k)\lesssim k^{-\tau}\), the crossing equation gives \(k_n=(\log n/\log\log n)(1+o(1))\), and therefore
\[
\sup_{D\ge d}
\bigl\|T_{d\to D}\hat f_{n,d,k_n}^{\mathrm{inv}}-f_D^\star\bigr\|_D^2
\lesssim
\left(\frac{\log n}{\log\log n}\right)^{-\tau}.
\tag{D.9}\label{eq:graph-polynomial-upper}
\]

\paragraph{Exponential source decay.}
For \(B(k)\lesssim e^{-a k^\nu}\), the balancing equation is
\[
a k_n^\nu+k_n\log k_n+O_r(k_n\log\log k_n)=\log n.
\tag{D.10}\label{eq:graph-exp-balance}
\]
Solving it gives
\[
B(k_n)+\frac{q_{\mathrm{graph}}(k_n)}n
\asymp
\begin{cases}
n^{-1+o(1)}, & \nu>1,\\
\exp\!\left[-a\left(\frac{\log n}{\log\log n}\right)^\nu(1+o(1))\right], & 0<\nu\le1.
\end{cases}
\tag{D.11}\label{eq:graph-exp-rates}
\]
The second line is critical at \(\nu=1\) and subcritical for \(0<\nu<1\). Equations~\eqref{eq:set-polynomial-upper},~\eqref{eq:set-exp-rates},~\eqref{eq:graph-polynomial-upper}, and~\eqref{eq:graph-exp-rates} prove Theorem~\ref{thm:size-generalization-oracle}.

\section{Proof of Theorem~\ref{thm:minimax-size-generalization}}
\label{app:proof-minimax-size-generalization}

We now show that the preceding rates cannot be improved. The lower bound is stated under the nondegeneracy condition in the theorem, which requires the invariant coordinates counted by representation stability to correspond to statistically distinguishable local perturbations. This is the standard condition that prevents the loss or data distribution from collapsing nominally different invariant directions.

\subsection{Minimax risk and local invariant submodels}

Let \(\mathfrak F_{\mathsf X}(B)\) be the class of size-consistent target families on domain \(\mathsf X\) satisfying exact transfer and the source envelope \(B\). An estimator observes \(n\) samples at size \(d\) and returns functions \(\hat f_D\) for \(D\ge d\). Define its worst-case risk by
\[
\mathfrak R_n(\mathsf X,B)
\coloneqq
\inf_{\hat{\boldsymbol f}}
\sup_{d\ge1}
\sup_{\boldsymbol f^\star\in\mathfrak F_{\mathsf X}(B)}
\mathbb E\!\left[
\sup_{D\ge d}\|\hat f_D-f_D^\star\|_D^2
\right].
\tag{E.1}\label{eq:minimax-risk}
\]

\paragraph{Minimax nondegeneracy.}
For every interaction order \(k\), take a size \(d_k\) in the stable range: \(d_k\ge k\) for sets and \(d_k\ge2k\) for graphs. The nondegeneracy assumption provides \(m_k\ge c_0q_{\mathsf X}(k)\) orthonormal invariant directions \(\psi_1,\ldots,\psi_{m_k}\), paired coherently across larger sizes, and an admissible baseline \(f_0\) such that the hypercube
\[
f_\omega=f_0+u\sum_{j=1}^{m_k}\omega_j\psi_j,
\qquad
\omega\in\{-1,1\}^{m_k},
\tag{E.2}\label{eq:local-hypercube}
\]
belongs to the source class whenever \(m_ku^2\le c_1B(k)\). It also requires the neighboring statistical experiments to satisfy
\[
\mathrm{KL}\!\left(P_\omega^{\otimes n},P_{\omega^{(j)}}^{\otimes n}\right)
\le c_2nu^2,
\tag{E.3}\label{eq:local-kl}
\]
where \(\omega^{(j)}\) differs from \(\omega\) only in coordinate \(j\). For regression, one obtains~\eqref{eq:local-kl} by adding Gaussian noise to \(Y=f_\omega(X)\). For density estimation, it follows from a bounded local density or exponential-family perturbation around an invariant density bounded away from zero. These constructions also ensure nonnegativity and normalization of every density in~\eqref{eq:local-hypercube}.

\subsection{Assouad lower bound}

Set
\[
u_k^2=c_3\min\!\left\{\frac{B(k)}{m_k},\frac1n\right\},
\]
with \(c_3\) sufficiently small. Then the hypercube is contained in the source class, and the KL divergence between neighboring vertices is bounded by a universal constant. Since orthonormality gives \(\|f_\omega-f_{\omega^{(j)}}\|_{d_k}^2=4u_k^2\), Assouad's lemma yields
\[
\inf_{\hat f}
\sup_{\omega\in\{-1,1\}^{m_k}}
\mathbb E_\omega\|\hat f-f_\omega\|_{d_k}^2
\ge c_4m_ku_k^2
\ge c_5\min\!\left\{B(k),\frac{q_{\mathsf X}(k)}n\right\}.
\tag{E.4}\label{eq:assouad-lower}
\]
The size-generalization loss contains the choice \(D=d_k\), so its supremum is at least the in-domain loss in~\eqref{eq:assouad-lower}. Consequently,
\[
\mathfrak R_n(\mathsf X,B)
\ge
c_5\sup_{k\ge1}
\min\!\left\{B(k),\frac{q_{\mathsf X}(k)}n\right\}.
\tag{E.5}\label{eq:general-minimax-lower}
\]
This observation is important: transfer cannot remove a lower bound already present at the training size.

Because \(B(k)\) decreases and \(q_{\mathsf X}(k)\) increases, the supremum in~\eqref{eq:general-minimax-lower} is attained, up to constants and lower-order factors, where \(B(k)\asymp q_{\mathsf X}(k)/n\). Under the regular source envelopes considered here,
\[
\sup_k\min\!\left\{B(k),\frac{q_{\mathsf X}(k)}n\right\}
\asymp
\inf_k\left\{B(k)+\frac{q_{\mathsf X}(k)}n\right\}.
\tag{E.6}\label{eq:upper-lower-equivalence}
\]
The right-hand side is precisely the quantity controlled by the upper oracle inequality~\eqref{eq:uniform-oracle}.

\subsection{Evaluation of the lower bound}

For polynomial source decay, substituting the set and graph representation asymptotics into the crossing condition gives, respectively,
\[
(\log n)^{-\tau(r+1)/r}
\qquad\text{and}\qquad
\left(\frac{\log n}{\log\log n}\right)^{-\tau}.
\tag{E.7}\label{eq:polynomial-lower-rates}
\]
For exponential source decay, the same crossing equations give exactly the three set regimes in~\eqref{eq:set-exp-rates} and the two graph regimes in~\eqref{eq:graph-exp-rates}.

These lower bounds match the upper bounds in Appendix~\ref{app:proof-scaling-laws}, up to constants and lower-order factors. They apply to both regression and density estimation through the local submodels above, and they already hold at a stable training size. Therefore no estimator can uniformly improve any set or graph rate in Table~\ref{tab:main-minimax-rates}, completing the proof of Theorem~\ref{thm:minimax-size-generalization}.

\section{Auxiliary finite-dimensional estimation bound}
\label{app:proof-invariant-m-estimation}

For completeness, we derive the finite-dimensional estimation term in~\eqref{eq:finite-oracle} from restricted curvature and score concentration. Empirical optimality and strong convexity reduce estimation error to the empirical gradient at the best population fit in the finite model; no statistical rate is assumed.

\subsection{Restricted gradients on the invariant model}

Fix \(d\in\mathbb{N}\). Let \(\Theta_d^{G_d}\) be a finite-dimensional invariant parameter space and let \(\theta_d^\star\in\Theta_d^{G_d}\) denote its population-risk minimizer. In this section, \(\theta_d^\star\) is the best fit within the finite model (the analogue of \(f_{d,k}^{\circ}\)), and need not equal the unrestricted target. We write
\[
p_d^G=\dim(\Theta_d^{G_d}).
\]
For simplicity of notation, we identify \(\Theta_d^{G_d}\) locally with a Euclidean space of dimension \(p_d^G\). Equivalently, if \(\Theta_d^{G_d}\) is a smooth model, the gradient below is understood in any normalized local coordinate chart around \(\theta_d^\star\). Thus, \(\nabla_G \hat L_{n,d}(\theta_d^\star)\) denotes the gradient of the empirical risk restricted to the invariant model.

We work on the restricted-curvature event on which the empirical risk
\[
\hat L_{n,d}(\theta)=\frac1n\sum_{i=1}^n \ell_d(Z_i,\theta)
\]
is \(\alpha\)-strongly convex on \(\Theta_d^{G_d}\), with \(\alpha\) independent of \(d\). This is a primitive curvature property, not an estimation-rate assumption. For squared-loss regression in an orthonormal feature basis, the Hessian is the empirical feature Gram matrix. If the whitened feature vector is uniformly sub-Gaussian, standard covariance concentration gives this event with probability at least \(1-\delta\) once \(n\gtrsim p_d^G+\log(1/\delta)\). For a regular likelihood or smooth proper loss, the same conclusion follows from a uniformly positive population Hessian and concentration of the restricted empirical Hessian. Pointwise strong convexity of the loss is a stronger sufficient condition, but is not required.

We will use the following standard concentration bound for the restricted empirical gradient.

\begin{lemma}[Restricted score concentration]
\label{lem:restricted-gradient-concentration}
Suppose the restricted score at \(\theta_d^\star\) is sub-Gaussian with variance proxy at most \(\sigma^2\). That is, writing
\[
Y_i\coloneqq \nabla_G \ell_d(Z_i,\theta_d^\star)\in\mathbb{R}^{p_d^G},
\]
we have \(\mathbb{E}Y_i=0\) and, for every unit vector \(v\in\mathbb{R}^{p_d^G}\),
\[
\mathbb{E}\exp(\lambda\langle v,Y_i\rangle)
\le
\exp\left(\frac{\lambda^2\sigma^2}{2}\right)
\qquad
\forall \lambda\in\mathbb{R}.
\]
Then there exists an absolute constant \(C_0>0\) such that, for every \(\delta\in(0,1)\), with probability at least \(1-\delta\),
\[
\left\|
\nabla_G \hat L_{n,d}(\theta_d^\star)
\right\|_2^2
=
\left\|
\frac1n\sum_{i=1}^n Y_i
\right\|_2^2
\le
C_0\sigma^2
\frac{p_d^G+\log(1/\delta)}{n}.
\]
\end{lemma}

\begin{proof}
Let \(p=p_d^G\) and define
\[
\bar Y_n\coloneqq \frac1n\sum_{i=1}^n Y_i .
\]
Since the \(Y_i\)'s are independent, mean-zero, and \(\sigma^2\)-sub-Gaussian, the average \(\bar Y_n\) is sub-Gaussian with variance proxy \(\sigma^2/n\). Hence for every unit vector \(v\in\mathbb{R}^p\),
\[
\mathbb{E}\exp(\lambda\langle v,\bar Y_n\rangle)
\le
\exp\left(\frac{\lambda^2\sigma^2}{2n}\right).
\]

Let \(\mathcal N\) be a \(1/2\)-net of the unit sphere \(\mathbb S^{p-1}\) with \(|\mathcal N|\le 5^p\). For every \(u\in\mathbb{R}^p\),
\[
\|u\|_2\le 2\max_{v\in\mathcal N}\langle v,u\rangle .
\]
Therefore, for any \(t>0\),
\[
\mathbb{P}(\|\bar Y_n\|_2\ge 2t)
\le
\mathbb{P}\left(\max_{v\in\mathcal N}\langle v,\bar Y_n\rangle\ge t\right).
\]
By a union bound and the one-dimensional sub-Gaussian tail bound,
\[
\mathbb{P}\left(\max_{v\in\mathcal N}\langle v,\bar Y_n\rangle\ge t\right)
\le
|\mathcal N|
\exp\left(-\frac{nt^2}{2\sigma^2}\right)
\le
\exp\left(p\log 5-\frac{nt^2}{2\sigma^2}\right).
\]
Taking
\[
t=\sigma\sqrt{\frac{2(p\log 5+\log(1/\delta))}{n}},
\]
we obtain, with probability at least \(1-\delta\),
\[
\|\bar Y_n\|_2
\le
2\sigma\sqrt{\frac{2(p\log 5+\log(1/\delta))}{n}}.
\]
Squaring both sides yields
\[
\|\bar Y_n\|_2^2
\le
8\sigma^2
\frac{p\log 5+\log(1/\delta)}{n}.
\]
Thus the claim holds with an absolute constant \(C_0\).
\end{proof}

\subsection{Proof of the estimation bound}

Let
\[
\hat\theta_{n,d}^{G}
\in
\arg\min_{\theta\in\Theta_d^{G_d}}\hat L_{n,d}(\theta).
\]
Since \(\theta_d^\star\in\Theta_d^{G_d}\), empirical optimality gives
\[
\hat L_{n,d}(\hat\theta_{n,d}^{G})
\le
\hat L_{n,d}(\theta_d^\star).
\]
On the intersection of the restricted-curvature event and the score-concentration event, \(\hat L_{n,d}\) is \(\alpha\)-strongly convex on \(\Theta_d^{G_d}\), and hence
\[
\hat L_{n,d}(\hat\theta_{n,d}^{G})
\ge
\hat L_{n,d}(\theta_d^\star)
+
\left\langle
\nabla_G \hat L_{n,d}(\theta_d^\star),
\hat\theta_{n,d}^{G}-\theta_d^\star
\right\rangle
+
\frac{\alpha}{2}
\left\|
\hat\theta_{n,d}^{G}-\theta_d^\star
\right\|_2^2.
\]
Combining the previous two displays gives
\[
0
\ge
\left\langle
\nabla_G \hat L_{n,d}(\theta_d^\star),
\hat\theta_{n,d}^{G}-\theta_d^\star
\right\rangle
+
\frac{\alpha}{2}
\left\|
\hat\theta_{n,d}^{G}-\theta_d^\star
\right\|_2^2.
\]
By Cauchy--Schwarz,
\[
0
\ge
-
\left\|
\nabla_G \hat L_{n,d}(\theta_d^\star)
\right\|_2
\left\|
\hat\theta_{n,d}^{G}-\theta_d^\star
\right\|_2
+
\frac{\alpha}{2}
\left\|
\hat\theta_{n,d}^{G}-\theta_d^\star
\right\|_2^2.
\]
If \(\hat\theta_{n,d}^{G}=\theta_d^\star\), the desired bound is immediate. Otherwise, dividing by
\(\|\hat\theta_{n,d}^{G}-\theta_d^\star\|_2\) gives
\[
\left\|
\hat\theta_{n,d}^{G}-\theta_d^\star
\right\|_2
\le
\frac{2}{\alpha}
\left\|
\nabla_G \hat L_{n,d}(\theta_d^\star)
\right\|_2.
\]
Taking each event with failure probability \(\delta/2\), squaring, applying Lemma~\ref{lem:restricted-gradient-concentration}, and using a union bound, we obtain with probability at least \(1-\delta\),
\[
\left\|
\hat\theta_{n,d}^{G}-\theta_d^\star
\right\|_2^2
\le
\frac{4}{\alpha^2}
\left\|
\nabla_G \hat L_{n,d}(\theta_d^\star)
\right\|_2^2
\le
C
\frac{\sigma^2}{\alpha^2}
\frac{p_d^G+\log(1/\delta)}{n},
\]
for an absolute constant \(C>0\). This is the finite-dimensional estimation bound used in~\eqref{eq:finite-oracle}. When the coordinates are orthonormal in \(\mathcal H_d\), Euclidean parameter distance equals function-space distance. If the population best fit additionally satisfies \(\|f_{d,k}^{\circ}-f_d^\star\|_d^2\le C_{\mathrm{app}}B_d(k)\), the squared triangle inequality gives~\eqref{eq:finite-oracle} directly.

If
\[
p^G\coloneqq \sup_d p_d^G<\infty,
\]
then \(p_d^G\le p^G\) for every \(d\), and therefore
\[
\left\|
\hat\theta_{n,d}^{G}-\theta_d^\star
\right\|_2^2
\le
C
\frac{\sigma^2}{\alpha^2}
\frac{p^G+\log(1/\delta)}{n}.
\]

\subsection{Proof of the transfer bound}

Fix a source dimension \(d\). Suppose that for every \(D\ge d\),
\[
\Phi_{d\to D}\theta_d^\star=\theta_D^\star,
\]
and that \(\Phi_{d\to D}\) is distance preserving on the invariant model:
\[
\left\|
\Phi_{d\to D}\theta-\Phi_{d\to D}\theta'
\right\|_2
=
\left\|
\theta-\theta'
\right\|_2
\qquad
\forall \theta,\theta'\in\Theta_d^{G_d}.
\]
Then, for every \(D\ge d\),
\[
\begin{aligned}
\left\|
\Phi_{d\to D}\hat\theta_{n,d}^{G}
-
\theta_D^\star
\right\|_2
&=
\left\|
\Phi_{d\to D}\hat\theta_{n,d}^{G}
-
\Phi_{d\to D}\theta_d^\star
\right\|_2 \\
&=
\left\|
\hat\theta_{n,d}^{G}
-
\theta_d^\star
\right\|_2.
\end{aligned}
\]
Taking the supremum over \(D\ge d\), we get
\[
\sup_{D\ge d}
\left\|
\Phi_{d\to D}\hat\theta_{n,d}^{G}
-
\theta_D^\star
\right\|_2^2
=
\left\|
\hat\theta_{n,d}^{G}
-
\theta_d^\star
\right\|_2^2.
\]
Combining this identity with the estimation bound above yields
\[
\sup_{D\ge d}
\left\|
\Phi_{d\to D}\hat\theta_{n,d}^{G}
-
\theta_D^\star
\right\|_2^2
\le
C
\frac{\sigma^2}{\alpha^2}
\frac{p_d^G+\log(1/\delta)}{n}
\]
with probability at least \(1-\delta\). If \(p^G=\sup_d p_d^G<\infty\), the right-hand side is bounded by
\[
C
\frac{\sigma^2}{\alpha^2}
\frac{p^G+\log(1/\delta)}{n}.
\]
This completes the auxiliary transfer calculation.

\section{Equivariant Field Specialization for Sets}
\label{app:proof-score-matching-sets}

Score matching and denoising replace invariant scalar functions by permutation-equivariant vector fields, but the proof is otherwise unchanged. Let \(\{\psi_{d,j}\}_{j\ge1}\) be an orthonormal equivariant basis and write
\[
s_d^\star=\sum_{j\ge1}\theta_{d,j}^\star\psi_{d,j},
\qquad
\|\hat s_d-s_d^\star\|_d^2
=\sum_{j\ge1}(\hat\theta_{d,j}-\theta_{d,j}^\star)^2.
\]
If \(\mathcal I_{d,k}\) contains the features of degree at most \(k\), then its stable dimension and the omitted coefficient energy are
\[
m_r(k)=|\mathcal I_{d,k}|
=\exp\!\left(C_rk^{r/(r+1)}+o\!\left(k^{r/(r+1)}\right)\right),
\qquad
B_d(k)=\sum_{j\notin\mathcal I_{d,k}}(\theta_{d,j}^\star)^2.
\]
Orthogonality and the auxiliary finite-dimensional bound in Appendix~\ref{app:proof-invariant-m-estimation} give, with probability at least \(1-\delta\),
\[
\|\hat s_{n,d}^{(k)}-s_d^\star\|_d^2
\le
B_d(k)+
C\frac{\sigma^2}{\alpha^2}
\frac{m_r(k)+\log(1/\delta)}{n}.
\tag{G.1}\label{eq:set-score-oracle}
\]
Thus the scalar proof applies verbatim. Uniform polynomial decay \(B_d(k)\lesssim k^{-\tau}\) gives \((\log n)^{-\tau(r+1)/r}\). Under uniform exponential decay \(B_d(k)\lesssim e^{-ak^\nu}\), with \(\beta=r/(r+1)\),
\[
\inf_k\left\{B_d(k)+\frac{m_r(k)}n\right\}
\asymp
\begin{cases}
n^{-1+o(1)}, & \nu>\beta,\\
n^{-a/(a+C_r)+o(1)}, & \nu=\beta,\\
\exp\!\left[-aC_r^{-\nu/\beta}(\log n)^{\nu/\beta}(1+o(1))\right], & 0<\nu<\beta.
\end{cases}
\tag{G.2}\label{eq:set-score-rates}
\]
Because the equivariant transfer operator is isometric, these bounds hold simultaneously at every target size \(D\ge d\).

\section{Equivariant Field Specialization for Graphs}
\label{app:proof-score-matching-graphs}

Let \(\mathcal J_{d,k}\) contain the degree-at-most-\(k\) equivariant graph features, and write \(M_r(k)=|\mathcal J_{d,k}|\). For fixed channel dimension \(r\),
\[
M_r(k)=\exp\!\left(k\log k+O_r(k\log\log k)\right),
\qquad
B_d^{\mathrm{graph}}(k)
=\sum_{j\notin\mathcal J_{d,k}}(\theta_{d,j}^\star)^2.
\]
The same orthogonal decomposition and finite-dimensional estimate yield
\[
\|\hat s_{n,d}^{(k)}-s_d^\star\|_d^2
\le
B_d^{\mathrm{graph}}(k)+
C\frac{\sigma^2}{\alpha^2}
\frac{M_r(k)+\log(1/\delta)}{n}.
\tag{H.1}\label{eq:graph-score-oracle}
\]
Consequently, polynomial decay \(B_d^{\mathrm{graph}}(k)\lesssim k^{-\tau}\) gives \((\log n/\log\log n)^{-\tau}\), while exponential decay \(B_d^{\mathrm{graph}}(k)\lesssim e^{-ak^\nu}\) gives
\[
\inf_k\left\{B_d^{\mathrm{graph}}(k)+\frac{M_r(k)}n\right\}
\asymp
\begin{cases}
n^{-1+o(1)}, & \nu>1,\\
\exp\!\left[-a\left(\dfrac{\log n}{\log\log n}\right)^\nu(1+o(1))\right], & 0<\nu\le1.
\end{cases}
\tag{H.2}\label{eq:graph-score-rates}
\]
The channel dimension affects only the lower-order term in \(M_r(k)\), so it does not enter these first-order rates. Isometric equivariant transfer again makes every bound simultaneous over \(D\ge d\).

\section{Experiments}

The experiments isolate two complementary parts of the theory rather than attempting a broad architectural benchmark. The first asks whether an invariant representation learned at one graph size retains its meaning at larger sizes. The second asks what remains when transfer is exact but the target must be approximated by a low-degree truncation. Together they separate representation-induced transfer error from approximation error due to the source tail.

\subsection{Subgraph counting}
We include a small synthetic experiment illustrating the role of invariant low-degree structure in transfer across graph sizes. The purpose of the experiment is not to benchmark a new architecture, but to visualize the basic mechanism behind our theory: a low-degree invariant statistic has a stable meaning across graph sizes, while a non-invariant coordinate polynomial does not transfer canonically.

We generate graphs from an Erdős--Rényi model with a graph-dependent edge probability. For each sample, we draw
\[
p \sim \mathrm{Unif}(0.15,0.85),
\qquad
A \sim G(d,p),
\]
where \(A\in\{0,1\}^{d\times d}\) is the adjacency matrix of an undirected graph with no self-loops. The response is the population triangle density $y \approx p^3$.
Thus, the target is a low-degree graph moment, but the learner observes only the sampled graph \(A\), not the latent parameter \(p\). We train all models only on graphs with \(d=10\) nodes and evaluate them on larger graphs with
\[
D \in \{25,30,35,40,50,60,75,100,125,150,175,200\}.
\]

We compare two predictors. The first is an invariant polynomial model using normalized graph statistics
\[
\phi_{\mathrm{inv},d}(A)
=
\big(E_d(A),W_d(A),T_d(A)\big),
\]
where \(E_d(A)\) is the edge density, \(W_d(A)\) is the normalized wedge density, and \(T_d(A)\) is the normalized triangle density. A linear model is fit on these features at \(d=10\), and transfer to size \(D\) is performed by evaluating the same normalized invariant statistics on the larger graph:
\[
\widehat f_{\mathrm{inv},D}(A)
=
\widehat\beta_0
+
\widehat\beta_1 E_D(A)
+
\widehat\beta_2 W_D(A)
+
\widehat\beta_3 T_D(A).
\]
This gives a direct size-transfer rule because the features are invariant graph moments whose interpretation is stable across graph sizes.

The second predictor is a non-invariant degree-three polynomial baseline. At the training size \(d=10\), we use the raw triangle indicators
\[
A_{ij}A_{ik}A_{jk},
\qquad
1\le i<j<k\le 10,
\]
as features and fit a linear model. Such a model assigns separate coefficients to labeled triangles and therefore has no canonical extension to graphs with more than ten nodes. To give this baseline a favorable transfer rule, we average it over random induced subgraphs: for a test graph on \(D\) nodes, we sample subsets \(S_1,\dots,S_M\subseteq [D]\) of size ten and define
\[
\widehat f_{\mathrm{raw},D}^{\mathrm{avg}}(A)
=
\frac{1}{M}
\sum_{m=1}^M
\widehat f_{\mathrm{raw},10}(A_{S_m\times S_m}).
\]
This is a Monte Carlo symmetrization of the raw non-invariant predictor. It deliberately gives the baseline an advantageous test-time correction: without averaging over induced subgraphs, its coordinates refer only to the ten training labels and do not define a predictor on a larger graph. The comparison therefore measures whether invariance is built into the learned representation or reconstructed afterwards by explicit averaging.

Figure~\ref{fig:experiments}(a) shows the test mean-squared error as the test graph size increases. The invariant polynomial model transfers directly and improves with graph size, since its normalized moments use all available nodes and become more accurate estimates of the underlying population quantities. The raw model also improves after subgraph averaging, but remains consistently worse despite this favorable correction. Thus, low polynomial degree alone is not the operative property: transfer comes from features whose invariant meaning is stable across sizes. The experiment also clarifies that larger test graphs need not be intrinsically harder; once the representation is size-consistent, additional nodes can provide more information about the same latent graph statistics.

\subsection{Equivariant field transfer on weighted graphs}

We next illustrate equivariant field estimation in a controlled setting where the target is known analytically. Continuous weighted graphs are represented by symmetric matrices
\[
A \in \mathbb R^{d\times d},
\]
with zero diagonal and independent standard-Gaussian upper-triangular entries. We regress onto the permutation-equivariant field
\[
s_d^\star(A)
=
-A
-
\sum_{q=2}^{8}
b_q q d^{-q} A^{q-1}.
\]
The features \(q d^{-q}A^{q-1}\) are gradients of normalized invariant spectral statistics \(d^{-q}\operatorname{tr}(A^q)\). Their meaning is shared across graph sizes, and the same coefficient \(b_q\) therefore defines the same rule at every size. We compare a fast-decay target,
\[
b_q = 2\exp(-0.14q^2),
\]
with a slow-decay target,
\[
b_q = 2q^{-1/2},
\]
for which the omitted high-order components remain substantial.

For each regime, we fit the degree-three equivariant model
\[
\widehat s_d(A)
=
-A-\sum_{q=2}^{4}\widehat b_q q d^{-q}A^{q-1}
\]
by least squares on the off-diagonal field entries of 200 independently sampled graphs with \(d=10\). Transfer to size \(D\) reuses the learned coefficients \(\widehat b_q\) and evaluates the corresponding normalized features \(qD^{-q}A^{q-1}\). We test on 16 independent graphs at each
\[
D \in \{10,15,20,30,40,60,80\}.
\]
Because the common linear term \(-A\) is represented exactly, we report error relative to the energy of the nonlinear target component:
\[
\operatorname{RelMSE}_D
=
\frac{\sum_A\sum_{i<j}(\widehat s_D(A)_{ij}-s_D^\star(A)_{ij})^2}
{\sum_A\sum_{i<j}(s_D^\star(A)_{ij}+A_{ij})^2}.
\]
The full procedure is repeated for ten independent seeds, and the shaded bands in Figure~\ref{fig:experiments}(b) report 95\% confidence intervals across these repetitions. No parameters are refit or calibrated at the target sizes: the only size-dependent operation is evaluating the analytically normalized features at (D).

Figure~\ref{fig:experiments}(b) shows that the same learned rule remains stable across all tested sizes. The absence of systematic deterioration with (D) is the empirical signature of exact equivariant transfer in this example. What remains is truncation error: fast decay makes the degree-three approximation highly accurate, whereas slow decay leaves a persistent gap roughly two to three orders of magnitude larger. This separates the two roles appearing in the theory---normalization controls size mismatch, while coefficient decay determines how much signal lies beyond the fitted degree. The experiment is illustrative rather than a test of the asymptotic minimax rates, which additionally depend on the full growth of invariant feature spaces.

\begin{figure}[!t]
    \centering
    \begin{minipage}[t]{0.485\linewidth}
        \centering
        \includegraphics[width=\linewidth]{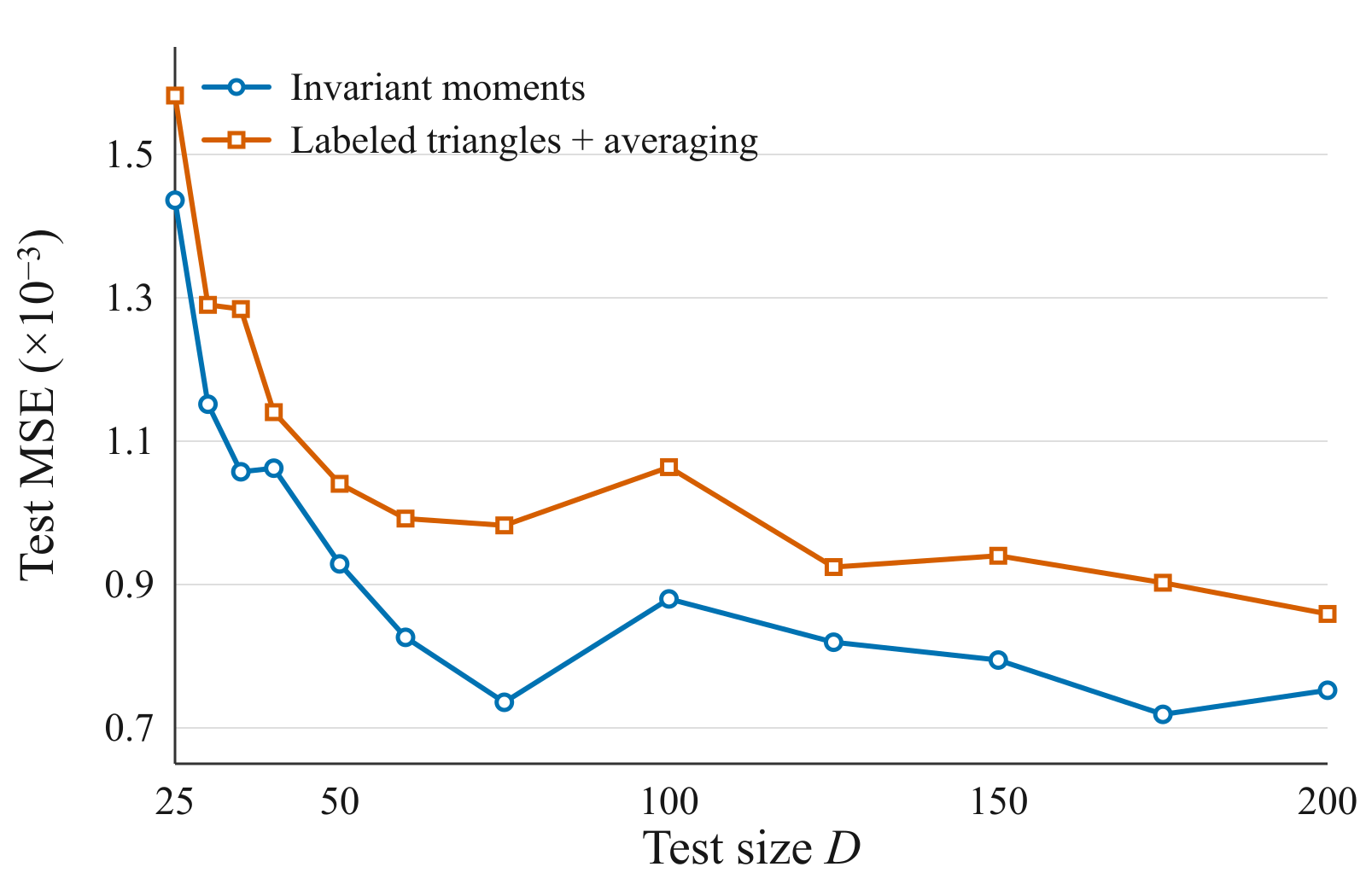}
        \vspace{0.2em}

        \small\textbf{(a)} Invariant versus symmetrized raw features.
    \end{minipage}\hfill
    \begin{minipage}[t]{0.485\linewidth}
        \centering
        \includegraphics[width=\linewidth]{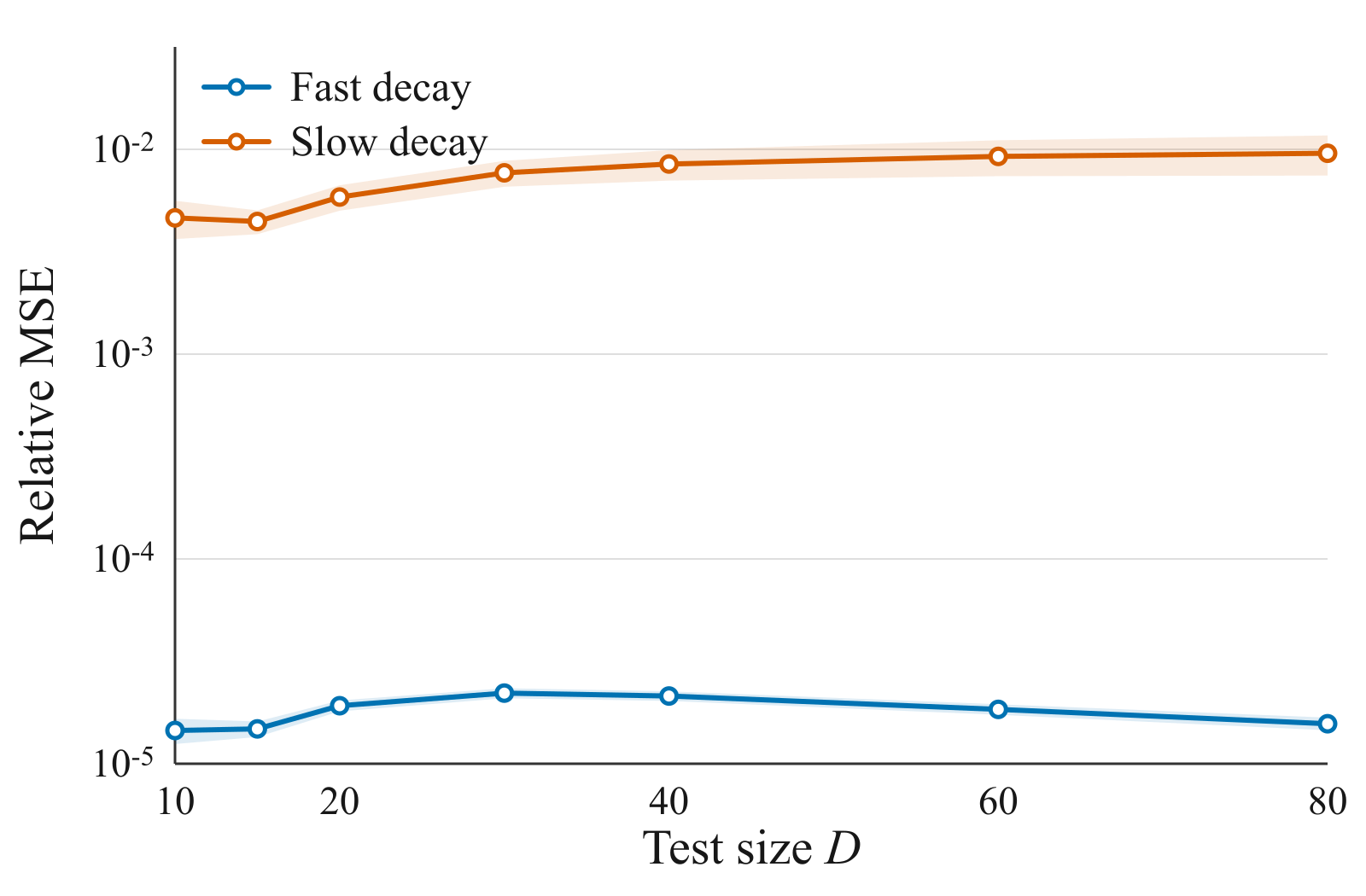}
        \vspace{0.2em}

        \small\textbf{(b)} Fast versus slow coefficient decay.
    \end{minipage}
    \caption{Two mechanisms governing transfer across graph sizes. \textbf{(a)} Transfer from 10-node Erd\H{o}s--R\'enyi graphs: normalized invariant moments transfer directly, whereas labeled triangle features require explicit subgraph averaging. \textbf{(b)} Exact equivariant transfer from 10-node weighted graphs using size-normalized features: curves show mean relative MSE over ten repetitions, with 95\% confidence intervals. Fast coefficient decay yields uniformly smaller truncation error.}
    \label{fig:experiments}
\end{figure}

\section{Notation}
\label{app:notation}

\begin{center}
\small
\setlength{\tabcolsep}{6pt}
\renewcommand{\arraystretch}{1.12}
\begin{tabular}{@{}p{0.25\linewidth}p{0.67\linewidth}@{}}
\toprule
\textbf{Symbol} & \textbf{Meaning} \\
\midrule
\(d,D\) & Training object size and target object size, with \(D\ge d\). \\
\(n\) & Number of training samples observed at size \(d\). \\
\(r\) & Element-feature dimension for sets or channel dimension for attributed graphs. \\
\(\mathcal X_d\) & Size-\(d\) input space: \(\mathbb R^{r\times d}\) for sets or \(\mathbb R^{r\times d\times d}\) for graphs. \\
\(Z_i\) & Unified observation: \(X_i\) for density estimation and \((X_i,Y_i)\) for regression. \\
\(\lambda_d\) & Standard Gaussian reference measure on \(\mathcal X_d\). \\
\(\mu_d,\rho_d\) & Input distribution and joint distribution of \((X,Y)\), respectively. \\
\(f_d^\star\) & Population target at size \(d\): a density relative to \(\lambda_d\) or a regression function. \\
\(\mathcal H_d\) & \(L^2(\lambda_d)\) for density estimation or \(L^2(\mu_d)\) for regression. \\
\(\|\cdot\|_d\) & The corresponding norm on \(\mathcal H_d\). \\
\(\mathcal F_d,\mathcal F_d^{\mathrm{inv}}\) & Statistical function class and its permutation-invariant subclass. \\
\(\ell_d,\widehat L_{n,d},L_d\) & Pointwise loss, empirical risk, and population risk. \\
\(S_d\) & Symmetric group acting by relabeling set elements or graph nodes. \\
\(\hat f_{n,d}^{\mathrm{inv}}\) & Invariant empirical risk minimizer learned at size \(d\). \\
\(T_{d\to D}\) & Exact coherent isometric transfer operator from size \(d\) to size \(D\). \\
\(k\) & Polynomial degree or interaction order. \\
\(\mathcal V_{d,k}\) & Invariant polynomials of degree at most \(k\) in \(\mathcal H_d\). \\
\(\mathcal F_{d,k}\) & Order-\(k\) approximation class; it also enforces density admissibility when required. \\
\(f_{d,k}^{\circ},\hat f_{n,d,k}^{\mathrm{inv}}\) & Population and empirical risk minimizers over \(\mathcal F_{d,k}\). \\
\(B_d(k),B(k)\) & Best order-\(k\) approximation error at size \(d\) and its supremum over sizes. \\
\(q_{\mathsf X}(k)\) & Stable dimension \(\sup_d\dim(\mathcal V_{d,k})\) for \(\mathsf X\in\{\mathrm{set},\mathrm{graph}\}\). \\
\(N_r(k),G_r(k)\) & Stable homogeneous degree-\(k\) invariant-feature counts for sets and graphs. \\
\(\tau\) & Polynomial source-decay exponent: \(B(k)\lesssim k^{-\tau}\). \\
\(a,\nu\) & Exponential source-decay constant and exponent: \(B(k)\lesssim e^{-ak^\nu}\). \\
\(\beta\) & Set representation-growth exponent \(r/(r+1)\). \\
\(C_r\) & Set growth constant \(\frac{r+1}{r}\bigl(r\zeta(r+1)\bigr)^{1/(r+1)}\). \\
\(\delta\) & Failure probability in high-probability bounds. \\
\(\mathfrak F_{\mathsf X}(B)\) & Size-consistent source class used in the minimax theorem. \\
\(\mathfrak R_n(\mathsf X,B)\) & Worst-case minimax size-generalization risk. \\
\(s_d^\star,\theta_{d,j}^\star\) & Equivariant field target and its coefficients in the optional score or denoising specialization. \\
\(\mathcal I_{d,k},\mathcal J_{d,k}\) & Degree-at-most-\(k\) equivariant feature indices for sets and graphs. \\
\(m_r(k),M_r(k)\) & Stable dimensions of the set and graph equivariant feature spaces. \\
\(\alpha,\sigma\) & Uniform curvature constant and gradient-noise scale in the auxiliary estimation bound. \\
\(C,c,c_i\) & Positive constants independent of \(n,d,D\); their values may change between displays. \\
\bottomrule
\end{tabular}
\end{center}


\end{document}

%% file: math_commands.tex
\usepackage{amsmath,amsfonts,bm}

\def\eqref#1{equation~\ref{#1}}

\def\1{\bm{1}}

\DeclareMathAlphabet{\mathsfit}{\encodingdefault}{\sfdefault}{m}{sl}
\SetMathAlphabet{\mathsfit}{bold}{\encodingdefault}{\sfdefault}{bx}{n}

